\documentclass[twoside,11pt]{article}

\usepackage{amsthm}
\usepackage[preprint]{jmlr2e}

\usepackage[utf8]{inputenc}
\usepackage[T1]{fontenc}
\usepackage{url}
\usepackage{booktabs}
\usepackage{multirow}
\usepackage{amsfonts}
\usepackage{nicefrac}
\usepackage{microtype}
\usepackage{xcolor}

\usepackage{amsmath}
\usepackage{mathtools}
\usepackage{bm}
\usepackage{bbm}

\usepackage{subcaption}
\usepackage{tikz}
\usetikzlibrary{positioning,arrows.meta,calc}
\usepackage{wrapfig}
\usepackage{float}

\usepackage{algorithm}
\usepackage{algorithmic}
\usepackage{cleveref}
\usepackage{aliascnt}

\makeatletter
\let\theorem\@undefined       \let\endtheorem\@undefined
\let\lemma\@undefined         \let\endlemma\@undefined
\let\proposition\@undefined   \let\endproposition\@undefined
\let\corollary\@undefined     \let\endcorollary\@undefined
\let\definition\@undefined    \let\enddefinition\@undefined
\let\remark\@undefined        \let\endremark\@undefined
\let\example\@undefined       \let\endexample\@undefined
\let\conjecture\@undefined    \let\endconjecture\@undefined
\let\axiom\@undefined         \let\endaxiom\@undefined
\let\c@theorem\@undefined     \let\thetheorem\@undefined
\let\c@example\@undefined     \let\theexample\@undefined
\makeatother

\newtheorem{theorem}{Theorem}[section]

\newaliascnt{lemma}{theorem}
\newtheorem{lemma}[lemma]{Lemma}
\aliascntresetthe{lemma}

\newaliascnt{proposition}{theorem}
\newtheorem{proposition}[proposition]{Proposition}
\aliascntresetthe{proposition}

\newaliascnt{corollary}{theorem}
\newtheorem{corollary}[corollary]{Corollary}
\aliascntresetthe{corollary}

\newaliascnt{definition}{theorem}
\newtheorem{definition}[definition]{Definition}
\aliascntresetthe{definition}

\newaliascnt{assumption}{theorem}
\newtheorem{assumption}[assumption]{Assumption}
\aliascntresetthe{assumption}

\newaliascnt{remark}{theorem}
\newtheorem{remark}[remark]{Remark}
\aliascntresetthe{remark}

\newaliascnt{example}{theorem}
\newtheorem{example}[example]{Example}
\aliascntresetthe{example}

\crefname{assumption}{Assumption}{Assumptions}
\Crefname{assumption}{Assumption}{Assumptions}
\crefname{lemma}{Lemma}{Lemmas}
\Crefname{lemma}{Lemma}{Lemmas}
\crefname{proposition}{Proposition}{Propositions}
\Crefname{proposition}{Proposition}{Propositions}
\crefname{corollary}{Corollary}{Corollaries}
\Crefname{corollary}{Corollary}{Corollaries}
\crefname{definition}{Definition}{Definitions}
\Crefname{definition}{Definition}{Definitions}
\crefname{remark}{Remark}{Remarks}
\Crefname{remark}{Remark}{Remarks}
\crefname{example}{Example}{Examples}
\Crefname{example}{Example}{Examples}

\newcommand{\cX}{\mathcal{X}}
\newcommand{\cY}{\mathcal{Y}}

\newcommand{\cU}{\mathcal{U}}
\newcommand{\cH}{\mathcal{H}}
\newcommand{\cP}{\mathcal{P}}

\newcommand{\cF}{\mathcal{F}}

\DeclareMathOperator{\supp}{supp}
\DeclareMathOperator{\rank}{rank}
\DeclareMathOperator{\range}{range}

\newcommand{\softmax}{\mathrm{softmax}}
\newcommand{\sigmoid}{\mathrm{sigmoid}}

\DeclareMathOperator*{\argmin}{arg\,min}
\DeclareMathOperator*{\argmax}{arg\,max}

\DeclareMathOperator{\E}{\mathbb{E}}
\DeclareMathOperator{\R}{\mathbb{R}}
\DeclareMathOperator{\1}{\mathbbm{1}}

\jmlrheading{1}{2026}{1-\pageref{LastPage}}{}{}{}{Kasra Jalaldoust}
\ShortHeadings{Abduction--Deduction Entanglement}{Jalaldoust \& Bareinboim}
\firstpageno{1}
\usepackage{lastpage}

\title{Abduction--Deduction Entanglement:\\ Domain Generalization via Representation Transplants}

\author{\name Kasra Jalaldoust \email kasra@cs.columbia.edu \\
\name Elias Bareinboim \email eb@cs.columbia.edu \\
        \addr Department of Computer Science\\
        Columbia University\\
        New York, NY 10027, USA}

\begin{document}
\maketitle

\begin{abstract}

Prediction models trained under the source distribution do not generalize well to a different target distribution. A valid inference about an unseen data distribution must be anchored by the invariance of certain causal mechanisms that generate the source and target data, however, these structural invariances are non-identifiable from the source data alone. Under mild causal assumptions about the data, we show that the optimal prediction in the target is in fact \emph{partially} identifiable by the source distribution. The result rests on a simple observation: In any domain, the optimal prediction can be factorized into what we call a pair of \emph{abduction and deduction maps}, where the abduction map makes inference about some unobserved variables (possibly confounders) from the observed variables and the deduction map predicts the label using both the observed and inferred quantities. Access to large source data pins down the optimal prediction, thus constrains the valid abduction--deduction ensembles that produce it --- a non-identifiability that we call the \emph{abduction--deduction entanglement}. To leverage this, we parameterize the constrained family using what we call a \emph{representation transplant}, that is a specific linear transformation in the representation space that manipulates the abduction content of the representation while retaining the deduction component. Invariance of the causal mechanism generating the label implies existence of an \emph{invariant deduction} map between source and target. Thus, we can search the space of plausible target distributions via a parametric transplant. We use this scheme in a learner--adversary game that, under an idealistic optimization, provably terminates with the learner having the minimax-optimal target prediction. Evaluations verify the theory, showing that the method is competitive in DG benchmarks.

\end{abstract}


\section{Introduction}
\label{sec:intro}

A prediction model trained on data from one distribution often degrades unexpectedly when the test distribution differs from the training one. From a learning-theoretic standpoint there is no a priori reason for it to be otherwise: without information relating source and target, no rule on the source controls the target risk~\citep{bendavid2010theory}. Challenges in domain generalization (DG) revolve around \emph{which} structural assumptions are useful to bridge source and target domains  while remaining plausible in modern ML applications.

Causal structural assumptions are promising inductive biases that allow principled probabilistic inference about unseen circumstances, and causal inference has inspired some of the important research directions on generalizability in machine learning. The driving principle in many of these work is that the statistical relationships between two data distributions are due to \emph{invariance} of certain underlying causal mechanisms that generate the data. In the case of prediction, a favored assumption is that the label $Y$ is generated by a causal mechanism that remains invariant, while the causal mechanisms that generate the covariates $X$ are subject to change. Under more restrictive assertions, such as absence of unobserved common causes (i.e., \emph{confounders}), one can derive the invariant causal prediction principle~\citep{peters2016causal}: there exists a \emph{representation} of covariates that learning to predict the label based on that representation yields a stable prediction across the domains. This inspires many of the work in DG that uses the heterogeneity in the source data to find invariant representations. Nonetheless, standardized DG benchmarks~\citep{gulrajani2021search,koh2021wilds} indicate that under a uniform evaluation protocol, the existing methods do not consistently improve over vanilla ERM and the failure is partially due to a mismatch between the assumed structure and real-world generalization necessities. We argue that an alternative notion of invariance can be productive for domain generalization, without further structural assumptions; below is a motivating example.

\subsection{A motivating example: moral questions and response prediction}
Consider a passive Q\&A panel: Subjective questions $X$ on different moral topics responded by individuals with different personalities $U$, with binary responses $Y\in\{y_0,y_1\}$ that represent action/inaction in a hypothetical scenario. The individuals answer to each question based on their moral values captured by $\tilde U := \psi(U) \in \{-1,0,1\}^3$, a \emph{coarsening} of the personality $U$ along three axes (\Cref{tab:moral-axes}), giving a total of $|\tilde\cU|=27$ different moral profiles.

\begin{table}[t]
\centering
\small
\begin{tabular}{lccc}
\toprule
\textbf{Axis} & \textbf{$+1$} & \textbf{$0$} & \textbf{$-1$} \\
\midrule
$U_{\mathrm{conseq}}$ & Consequentialist (only outcomes matter) & Balanced     & Deontologist (rules are intrinsic) \\
$U_{\mathrm{indv}}$   & Individualist (autonomy paramount)      & Balanced     & Collectivist (group welfare first) \\
$U_{\mathrm{care}}$   & Care/compassion-prioritizing            & Balanced  & Justice/impartiality-prioritizing \\
\bottomrule
\end{tabular}
\caption{The three moral-value axes constituting the coarsening $\tilde U \in \{-1, 0, +1\}^3$ of the underlying moral values $U$. Each axis ranges over a moral pole; $0$ denotes a balanced/mixed orientation. The convention $+1$ denotes ``embodies the named pole''.}\label{tab:moral-axes}
\end{table}

For instance, consider the following question $x$: \emph{``A scientist developing a drug that increases intelligence needs to test it on a volunteer who may suffer permanent mental impairment"}. The choices offered to the responder are $y_0$: \emph{``refuse to use the drug"}, and  $y_1$: \emph{``use it.''} A \emph{consequentialist} person ($\tilde u_{1}=+1$), weighing the societal benefit of a working drug against one volunteer's risk, may respond $y_1$; a \emph{deontologist} ($\tilde u_{1}=-1$), holding that one must not use a person as a means to ends, is likely to respond $y_0$. Additionally, we know that each responder also has different tendencies regarding \emph{which} questions they engage with: questions foregrounding \emph{freedom} or \emph{pleasure} attract more responses from \emph{individualists}, while questions foregrounding \emph{duty} or \emph{promises} attract more responses from deontologist--collectivist individuals.

Suppose one uses this data and trains a prediction model to approximate $P(y_1\mid x)$. By the law of total probability, the model admits the following factorization:
\begin{equation}
    P(y_1\mid x) \;=\; \sum_{\tilde u \in \tilde \cU}\;\underbrace{P(\tilde u\mid x)}_{\text{who might've responded to }x\text{?}}\,\cdot\,\underbrace{P(y_1\mid x,\tilde u)}_{\text{what response would such a person give?}}
\end{equation}
Any such prediction model can be viewed as an ensemble of two steps: an \emph{abduction} $P(\tilde U\mid X)$ that infers the moral values of the question's responder, and a \emph{deduction} $P(\tilde Y\mid X,\tilde{U})$ step that predicts what such an individual would respond. A familiar reader will recognize the same notion of abduction used in the 3-step counterfactual inference procedure---abduction, action, prediction~\citep{pearl2009causality}.

The factorization above is particularly useful to understand why a prediction fails under \emph{distribution shifts}. For instance, in the moral questions example, it is likely that the abduction step ``who might've responded to question $x$'' changes from one domain to the other; for instance, the change may be due to different moral value distributions across the populations, i.e., $P(\tilde U) \neq P^\star(\tilde{U})$, or a change in question--responder association, i.e., $P(X\mid \tilde U) \neq P^\star(X\mid \tilde{U})$, e.g., when in one population the individuals tend to engage with questions that challenges their views and in the other population individuals engage with questions even if it does not directly correspond to their moral stance. 

Notice that the prediction rule $P(Y\mid X)$ is identifiable from the question/response data distribution $P(X,Y)$. However, the true abduction and deduction remain unknown since for any random variable $\tilde U'$ there exists a valid factorization with abduction $P(\tilde U'\mid X)$ and deduction $P(Y\mid X,\tilde{U}')$. The two factors, however, are not \emph{independent} of each other: the ensemble must equal the prediction rule that is revealed by the data; we refer to this as the \emph{abduction--deduction entanglement}. But how can an \emph{entanglement} be useful for DG?

\subsection{Existing remedy: eliminate the entanglement}
A substantial line of work makes the DG problem tractable via structural assertions that \emph{remove} the entanglement altogether. Invariant Causal Prediction~\citep{peters2016causal} and Invariant Risk Minimization~\citep{arjovsky2019invariant} center around existence of a feature subset $X_I \subset X$ such that optimal prediction based on it is invariant across source and target, $P(Y\mid X_I)=P^\star(Y\mid X_I)$; this is especially true for the direct causes of $Y$. The learning objective becomes identifying $X_I$ from heterogeneous source data, e.g., diverse-enough domains or different interventional regimes. We defer a full discussion of this lineage to \Cref{sec:related-work}; here we read it through the abduction--deduction lens.

\paragraph{Relation to A/D.} Existence of an invariant predictor $X_I \subset X$ is interesting from abduction--deduction point-of-view: It is equivalent to invariance of both abduction and deduction for \emph{any} factorization, i.e., for all random variables $\tilde{U}'$, under some non-degeneracy assumptions\footnote{Assuming that the equalities are not accidental, e.g., we exclude the cases where a mismatch between the abduction and deduction maps cancels out through marginalization; this notion is closely related to the notion of causal faithfulness~\citep{spirtes2000causation,pearl2009causality} widely employed throughout causal machine learning.},
\begin{equation}
    P(Y\mid X_I) = P^\star(Y\mid X_I) {\iff} \begin{cases}
        P(\tilde{U}'\mid X_I) = P^\star(\tilde{U}'\mid X_I) &\text{ (abduction invariance)} \\
        P(Y\mid X_I,\tilde{U}') = P^\star(Y\mid X_I,\tilde{U}') &\text{ (deduction invariance)}
    \end{cases}
\end{equation}
In words, causal assumptions that make invariant prediction viable are implicitly escaping the entanglement by making it idle. Not only this may not be viable (either practically, or even conceptually) in certain DG tasks of interest, but also invariance of both abduction and deduction may not be necessary for the source data to inform prediction in the target; we discuss how deduction invariance is enough. To make the deduction parameterization tractable, however, we need a representation that exposes the abduction and deduction factors as linear readouts---which leads us to the role of pretrained foundation representations.

\subsection{Generalizability as a representation learning objective}
Another important objective in unsupervised ML is to learn data \emph{representations} that are useful for downstream tasks such as prediction/classification. Causal representation learning ~\citep{scholkopf2021toward,locatello2019challenging,hyvarinen2016unsupervised,khemakhem2020variational,vonkugelgen2021self,ahuja2023interventional,brehmer2022weakly,squires2023linear,lippe2022citris} aims to leverage data heterogeneity to identify the underlying causal variables and mechanism. The premise is that a ``causal" representation captures changes to the data distribution in \emph{disentangled} ways, and therefore, offers interpretability, robustness, and generalization for free to the downstream learner. In spirit, our discussion on abduction--deduction entanglement is related to the literature on disentangled representation learning, with a key distinction that we do not attempt to identify/disentangle the factors, rather use them to reason about a specified prediction task.

A complementary thread in modern ML attempts the same goal in downstream tasks by learning representations from large pooled datasets, a.k.a., the foundation models~\citep{bommasani2021opportunities,devlin2019bert,brown2020language,radford2021learning}. The rationale is that a pretrained image/text representation $\phi:\cX\to\R^d$ that has been exposed to very large sets of \emph{diverse} data from different domains would automatically develop the \emph{inductive biases} that are necessary to generalize to unseen distributions. 

\paragraph{Relation to A/D.} The foundation models perspective to learning generalizable representations is in fact sensible from the abduction--deduction perspective: Consider a moral question $x$ and a pretrained text embedding $\phi(x) \in \R^d$. Many different \emph{concepts} related to question $x$ are captured by its representation, and empirical evidence suggests that these concepts are often linearly encoded in pretrained representations, a.k.a., the linear representation hypothesis (LRH)---tracing back to early evidence of arithmetic structure in word embeddings~\citep{mikolov2013distributed} and feature-direction findings in vision-model interpretability~\citep{olah2020zoom}, and formalized as a named hypothesis by~\citet{park2024linear}. If these concepts encoded in $\phi(X)$, e.g., the realistic/idealistic framing of the question $x$, happens to \emph{correlate} well with the responder's moral value $\tilde{U}$, then the representation contains the information necessary to mimic the true abduction rule, e.g., it is possible that $P(\tilde{U}\mid x)$ is linearly expressible in $\phi(x)$. The LRH motivates---but does not formally imply---the softmax/sigmoid parameterization in \Cref{eq:abduction-param,eq:deduction-param}; the latter is our parametric working assumption, distinct from LRH as a phenomenon about pretrained representations. 

\subsection{Contribution} In this work, rather than removing entanglement, we embrace it; we parameterize the collection of abduction--deduction maps that are entangled by a source prediction. To make generalization viable, we assert \emph{invariant deduction}, i.e., we assume that there exists some unobserved variable $\tilde{U}$ for which $P(Y\mid X,\tilde{U}) = P^\star(Y\mid X,\tilde{U})$, and to make matters tractable, we assume that the abduction/deduction maps for $\tilde{U}$ are linearly encoded by $\phi(x)$, inspired by the linear representation hypothesis. Notably, invariant deduction is readily derived from invariance of the causal mechanism for the label, i.e., $f_Y(X,U) = f^\star_Y(X,U)$, \emph{provided} the coarsening is sufficient for $Y$ given $X$---that is, $P(Y\mid X, U) = P(Y\mid X, \tilde U)$, equivalently $Y\perp U\mid X, \tilde U$. Under an arbitrary coarsening these can come apart, since marginalizing $U$ over the fiber $\psi^{-1}(\tilde u)$ involves the domain-varying conditional $P^s(U\mid X,\tilde U)$. Invariance of $f_Y$ is commonly assumed in many causality-inspired work in DG. Despite the non-identifiability, we show that the source distribution $P(X,Y)$ partially identifies the target prediction $P^\star(Y\mid X)$. For prediction in the target, we use a learner--adversary game called \emph{Causal Robust Optimization} (CRO) whose solution attains the best worst-case target risk over causally plausible target distributions (\Cref{sec:transplant,sec:algorithm}). Furthermore, we show that the existence of linear invariant representations~\citep{arjovsky2019invariant} emerges as a structural consequence---a corollary of the parameterization---rather than the learning principle. Empirically, CRO shows competitive results based on unified evaluation on some of the standard DG benchmarks in DomainBed~\citep{gulrajani2021search}. Most notably, on the diagnostic Colored MNIST cell where source and target color--label correlations are anti-aligned, CRO achieves $76\%$ target accuracy while the established and externally verified baselines---ERM, IRM~\citep{arjovsky2019invariant}, GroupDRO~\citep{sagawa2020distributionally}, V-REx~\citep{krueger2021out}, CORAL~\citep{sun2016deep}, and Mixup~\citep{zhang2018mixup}---collapse to around $10\%$. CRO also achieves \emph{competitive} performance on real-world benchmarks (PACS, OfficeHome) where the deduction-invariance assumption is only approximate (\Cref{sec:experiments}).

\subsection{Preliminaries and notation}
We use uppercase $Z$ to denote random variables, and lowercase $z$ to denote their values; the matching calligraphic letter denotes the set of possible values, e.g., $x\in\cX$. We write $P(y_1\mid x)$ for $P(Y=1 \mid X=x)$; $\softmax$, $\sigmoid$ carry their standard definitions. 

We work with structural causal models (SCMs)~\citep{pearl2009causality} of the form,
\begin{equation*}
    U\sim P(u), \quad X\gets f_X(U), \quad Y\gets f_Y(X,U).
\end{equation*}
where $X,Y$ are observable covariates and label, and $U$ is the unobserved common cause, a.k.a., the confounder. A \emph{coarsening} of $U$ is any (typically discrete) variable $\tilde U := \psi(U) \in \tilde\cU$. In the multi-domain setting we observe $X,Y$ in $K$ source distributions $P^1, P^2, \dots, P^K$, and seek inference about an unseen target $P^\star$; in the appropriate context, $P^{+}$ refers to pooled source distribution, i.e., $P^{+}(x,y) = \frac{1}{K}\sum_{s=1}^K P^s(x,y)$.

A predictor is $g:=h\circ\phi$ with $\phi:\cX\to\R^d$ being a fixed pretrained representation and $h\in\cH$ a classification head that outputs class probabilities. The risk of $g$ under $P$ is denoted by $R_P(g):=\E_P[\ell(g(X),Y)]$, where $\ell$ is a specified loss function. The running choice in this paper is the convex surrogate cross-entropy $\ell(\hat p, y)=-\log \hat p_y$, used both in the algorithm and in our theoretical statements; we report $0\textnormal{-}1$ accuracy as the evaluation metric. 

\section{Partial Transportability Result}
\label{sec:transplant}

This section develops the paper's core machinery. We observe $K$ source distributions $P^1(X,Y),\dots,P^K(X,Y)$ and seek to predict under an unseen target distribution $P^\star(X,Y)$ with small target risk $R_{P^\star}(g) := \E_{P^\star}[\ell(g(X),Y)]$. Each domain induces the causal graph in \Cref{fig:SD-for-bow-graph}: $X,Y$ are observable, $U$ is an unobserved common cause of $X$ and $Y$, the response mechanism $f_Y$ is shared across domains, and the cross-domain variation enters through a domain-specific prior $\pi_s(U)$ over the unobserved $U$. We work with a discrete coarsening $\tilde U := \psi(U) \in \tilde\cU$ of the underlying $U$, used to parameterize the abduction--deduction factorization below, and a frozen pretrained representation $\phi:\cX\to\R^d$. Our construction has three parts: (i) an abduction--deduction factorization of the prediction rule into a domain-varying abduction map and a domain-invariant deduction map, parameterized as softmax/sigmoid readouts linear in $\phi$ (softmax-linear abduction and logit-linear deduction); (ii) a low-rank linear operator on $\phi$---the \emph{cross-domain transplant}---that swaps the abduction component of a representation across domains while preserving its deduction component; and (iii) a parameterized family of plausible target distributions derived from these two ingredients, against which we will optimize worst-case target risk in \Cref{sec:algorithm}.

\begin{wrapfigure}{r}{0.32\textwidth}
\centering
\vspace{-\baselineskip}
\begin{tikzpicture}[
    every node/.style={circle, draw, minimum size=0.7cm, inner sep=1pt},
    >={Stealth[length=1.6mm]}
]
    \node[dashed] (U) at (0,1.4) {$U$};
    \node[rectangle, fill=black, minimum size=0.25cm, inner sep=0pt, label=above:$s$] (S) at (-2.1,1) {};
    \node (X) at (-1.1,0) {$X$};
    \node (Y) at (1.1,0) {$Y$};

    \draw[->, dashed] (U) -- (X);
    \draw[->, dashed] (U) -- (Y);
    \draw[->] (X) -- (Y);
    \draw[->] (S) -- (X);
\end{tikzpicture}
\caption{Per-domain causal graph for the multi-domain setting: $X, Y$ are observable, $U$ is unobserved (dashed). The response mechanism $f_Y$ is shared across domains; the cross-domain variation enters through a domain-specific prior $\pi_s(U)$ over the unobserved $U$ (equivalently, through a domain-specific $f_X^s$).}
\label{fig:SD-for-bow-graph}
\vspace{-\baselineskip}
\end{wrapfigure}
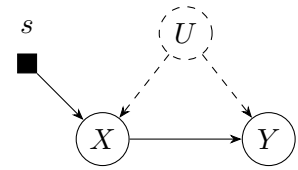

\subsection{Illustrative example: domain generalization}
\label{sec:illustrative-multi}

Continuing the moral-questions example from \Cref{sec:intro}: source data is collected from $K$ Q\&A panels, each with its own prior $\pi_s(\tilde U)$ over the $|\tilde\cU|=27$ moral-values triples (\Cref{tab:domain-priors}); the target $P^\star$ is an unrepresented panel whose composition we do not see. Concretely, three natural panels are a \emph{pragmatist} panel concentrated on consequentialist--individualist--care individuals, an \emph{idealist} panel concentrated on deontologist--collectivist--justice individuals, and a \emph{centrist} panel concentrated on the balanced (mixed) corner. Each panel induces the causal graph of \Cref{fig:SD-for-bow-graph}; the response mechanism $f_Y$ and the question library are shared, while the prior $\pi_s(\tilde U)$ varies across $s$. Within each source $s$, a sampled question $X$ is more likely to land on an individual whose $\tilde U$ aligns with the panel's mix, so the abduction map $P^s(\tilde u\mid x)$ varies with $s$. Given the same question and the same $\tilde U$, however, the response is the same across panels---a consequentialist's choice on a fixed question does not depend on which panel they were sampled from. \Cref{tab:sample-scenarios} shows one representative scenario per panel. We return to this example throughout the section as an illustration; the development itself is in the abstract notation of \Cref{fig:SD-for-bow-graph}.

\begin{table}[t]
\centering
\caption{Three illustrative source panels with per-axis moral-values priors $\pi_s(\tilde U)$, independent across the three axes within a panel. Each entry is the probability mass on $(+1,\,0,\,-1)$. Only the moral-values prior varies across panels; the scenario library $\cX$ and the response mechanism $P(Y\mid X,\tilde U)$ are shared.}
\label{tab:domain-priors}
\small
\begin{tabular}{lcccl}
\toprule
\textbf{Panel} & $\pi_s(\tilde U_{\mathrm{conseq}})$ & $\pi_s(\tilde U_{\mathrm{indv}})$ & $\pi_s(\tilde U_{\mathrm{care}})$ & \textbf{Modal moral values} \\
\midrule
Pragmatist & $(.80,\,.10,\,.10)$ & $(.80,\,.10,\,.10)$ & $(.80,\,.10,\,.10)$ & $(+1,+1,+1)$ \\
Idealist   & $(.10,\,.10,\,.80)$ & $(.10,\,.10,\,.80)$ & $(.10,\,.10,\,.80)$ & $(-1,-1,-1)$ \\
Centrist   & $(.10,\,.80,\,.10)$ & $(.10,\,.80,\,.10)$ & $(.10,\,.80,\,.10)$ & $(0,0,0)$ \\
\bottomrule
\end{tabular}
\end{table}

\begin{table}[t]
\centering
\small
\begin{tabular}{lcp{8.4cm}c}
\toprule
\textbf{Panel} & \textbf{Moral values $\tilde U$} & \textbf{Scenario} & $Y$ \\
\midrule
Centrist & $(+1,0,0)$ & \emph{Context}: A scientist developing a formula to increase intelligence needs to test it on a volunteer who may suffer a reduction of mental ability. \quad $A$: refuse to use the formula. \quad $B$: use the formula. & $B$ \\
\midrule
Pragmatist & $(0,+1,+1)$ & \emph{Context}: A best friend hosting a party asks you to bring alcohol, even though the guests are too young to drink it. \quad $A$: don't bring it. \quad $B$: bring it anyway. & $A$ \\
\midrule
Idealist & $(-1,-1,+1)$ & \emph{Context}: An engineer discovers a flaw in a product that could lead to disastrous consequences if released; not fixing it would delay the launch for months and cost the company a fortune. \quad $A$: report the flaw. \quad $B$: keep the information to myself. & $A$ \\
\bottomrule
\end{tabular}
\caption{One representative scenario per panel. The $\tilde U$ column gives the moral-values triple under which the response $Y$ was elicited for that scenario---a representative draw under the panel's prior, not necessarily the panel's modal triple (cf.\ \Cref{tab:domain-priors}). Each scenario presents a context and two alternatives $A$ and $B$.}
\label{tab:sample-scenarios}
\end{table}

\subsection{Abduction--deduction parameterization}
\label{sec:parameterization}

For each domain $s\in\{1,\dots,K,\star\}$, the law of total probability decomposes the prediction rule into an inner product
\begin{equation}\label{eq:AD-decomp}
    P^s(y\mid x) = \sum_{\tilde u \in \tilde \cU} P^s(\tilde u\mid x)\cdot P^s(y\mid x,\tilde u),
\end{equation}
which we write as
\begin{equation}\label{eq:AD-entanglement}
    P^s(y_1\mid x) = \vec{P}^s(\bm{\tilde u}\mid x)^\top\,\vec{P}^s(y_1\mid x,\bm{\tilde u}).
\end{equation}
\begin{definition}[Abduction and deduction maps]\label{def:abduction-deduction}
For domain $s\in\{1,\dots,K,\star\}$, write both maps as $r$-vectors indexed by $\tilde u\in\tilde\cU$:
\[
\vec{P}^s(\bm{\tilde u}\mid x) := \langle P^s(\tilde u\mid x)\rangle_{\tilde u\in\tilde\cU}^\top
\qquad\text{(\emph{abduction map})},
\]
\[
\vec{P}^s(y_1\mid x,\bm{\tilde u}) := \langle P^s(y_1\mid x,\tilde u)\rangle_{\tilde u\in\tilde\cU}^\top
\qquad\text{(\emph{deduction map})}.
\]
The abduction map infers the unobserved $\tilde U$ from the observed $X$; the deduction map predicts $Y$ given both $X$ and the inferred $\tilde U$, both under domain $s$.
\end{definition}

Within each domain, neither map is identifiable from data, but the two are \emph{entangled} by \Cref{eq:AD-entanglement}: their inner product is pinned to $P^s(y_1\mid x)$, which source data identifies.

\begin{remark}[Abduction--deduction entanglement]\label{rem:entanglement}
We refer to this constraint as the \emph{abduction--deduction entanglement}: the inner product $\vec{P}^s(\bm{\tilde u}\mid x)^\top \vec{P}^s(y_1\mid x,\bm{\tilde u})$ is pinned to $P^s(y_1\mid x)$ by \Cref{eq:AD-entanglement} while the two factors individually remain underdetermined. The paper's central move (\Cref{sec:plausible-target}) is to parameterize, rather than break, this entanglement.
\end{remark}

\paragraph{Shared deduction, varying abduction.} The structural premise of the multi-domain setting is that the deduction map is invariant across domains while the abduction map varies. Conditioning on both $X$ and $\tilde U$ pins down the response mechanism uniformly across domains; what differs across domains is which $\tilde U$ likely produced a given $X$. Formally,
\begin{equation}\label{eq:deduction-invariance}
    \vec{P}^1(y_1\mid x,\bm{\tilde u}) = \dots = \vec{P}^K(y_1\mid x,\bm{\tilde u}) = \vec{P}^\star(y_1\mid x,\bm{\tilde u}),
\end{equation}
while the abduction map $P^s(\tilde u\mid x)$ remains domain-specific.

\paragraph{Softmax/sigmoid parameterization linear in $\phi$.} We parameterize the abduction map as a softmax (multinomial-logit) readout and the deduction map as a sigmoid (logit-linear, or log-odds linear) readout, both linear in a frozen pretrained representation $\phi:\cX\to\R^d$. Fix a \emph{transplant dimension} $r$ for the parameterized coarsening (identified with $|\tilde\cU|$). The true support size of $\tilde U$ is generally unknown, so $r$ is a modeling choice; \Cref{fig:rank-vs-test} below stress-tests sensitivity to $r$. We posit matrices $A_1,\dots,A_K,A_\star\in\R^{d\times r}$ and $D\in\R^{d\times r}$ such that for every $s\in\{1,\dots,K,\star\}$,
\begin{align}
    \vec{P}^s(\bm{\tilde u}\mid x) &\approx \softmax(A_s^\top\phi(x)), \label{eq:abduction-param} \\
    \vec{P}^s(y_1\mid x,\bm{\tilde u}) &\approx \sigmoid(D^\top\phi(x)). \label{eq:deduction-param}
\end{align}

\emph{Indexing convention.} Fix an ordering of $\tilde\cU$. Both sides of \Cref{eq:abduction-param,eq:deduction-param} are $r$-vectors indexed by $\tilde u\in\tilde\cU$. The $\tilde u$-th entry on the LHS of \Cref{eq:abduction-param} is the scalar abduction probability $P^s(\tilde u\mid x)$; on the RHS, it is the $\tilde u$-th component of the softmax. For \Cref{eq:deduction-param}, the $\tilde u$-th entry on the LHS is the scalar conditional $P^s(y_1\mid x,\tilde u)$ (one Bernoulli probability per $\tilde u$); on the RHS, $\sigmoid$ is applied entry-wise to the $r$-vector $D^\top\phi(x)$, and its $\tilde u$-th entry is the parameterized value of $P^s(y_1\mid x,\tilde u)$. There is no $\tilde u$ on the RHS as a variable; the $\tilde u$-dependence lives in the entry index.

$A_s$ is the domain-$s$ abduction parameter and $D$ is the shared deduction parameter. The pair $(A_s, D)$ parameterizes the abduction--deduction factorization as linear-in-$\phi$ readouts under the softmax and sigmoid link functions. A trained head $h_s: \R^d \to [0,1]$ identified from domain-$s$ data, $(h_s\circ\phi)(x) \approx P^s(y_1\mid x)$, fixes the inner product in \Cref{eq:AD-entanglement} but leaves the factors $(A_s, D)$ underdetermined. The entanglement is, in the parameterized form, an underdetermined system constraining $(A_s, D)$ through $h_s$.

\subsection{Representation Transplants}
\label{sec:transplant-operator}

\begin{definition}[Transplant basis]\label{def:transplant-basis}
Given abduction parameters $A_1,\dots,A_K,A_\star\in\R^{d\times r}$ and deduction $D\in\R^{d\times r}$, a \emph{transplant basis} is any matrix $L \in \R^{d\times r}$ satisfying
\begin{equation}\label{eq:transplant-basis}
    A_s^\top L = I_r \;\;\forall s\in\{1,\dots,K,\star\}, \qquad D^\top L = 0.
\end{equation}
\end{definition}

A transplant basis simultaneously left-inverts every abduction map (including the unseen target's) and lies in the left-null-space of the deduction map, so the transplant operator below acts only on the abduction component of $\phi(x)$. Existence of $L$ is characterized in \Cref{prop:L-exists} of \Cref{app:proofs-L-exists}; $d\ge(K+2)r$ is necessary, and generic-position parameters in such $d$ suffice.

\begin{definition}[Transplant]\label{def:T-multi}
For $s,s'\in\{1,\dots,K,\star\}$,
\begin{equation}\label{eq:rep-transport-map}
    T_{s\to s'}: W \;\mapsto\; \big(I_d + L(A_s^\top - A_{s'}^\top)\big)\,W.
\end{equation}
\end{definition}

\begin{wrapfigure}{r}{0.45\textwidth}
\vspace{-\baselineskip}
\centering
\includegraphics[width=\linewidth]{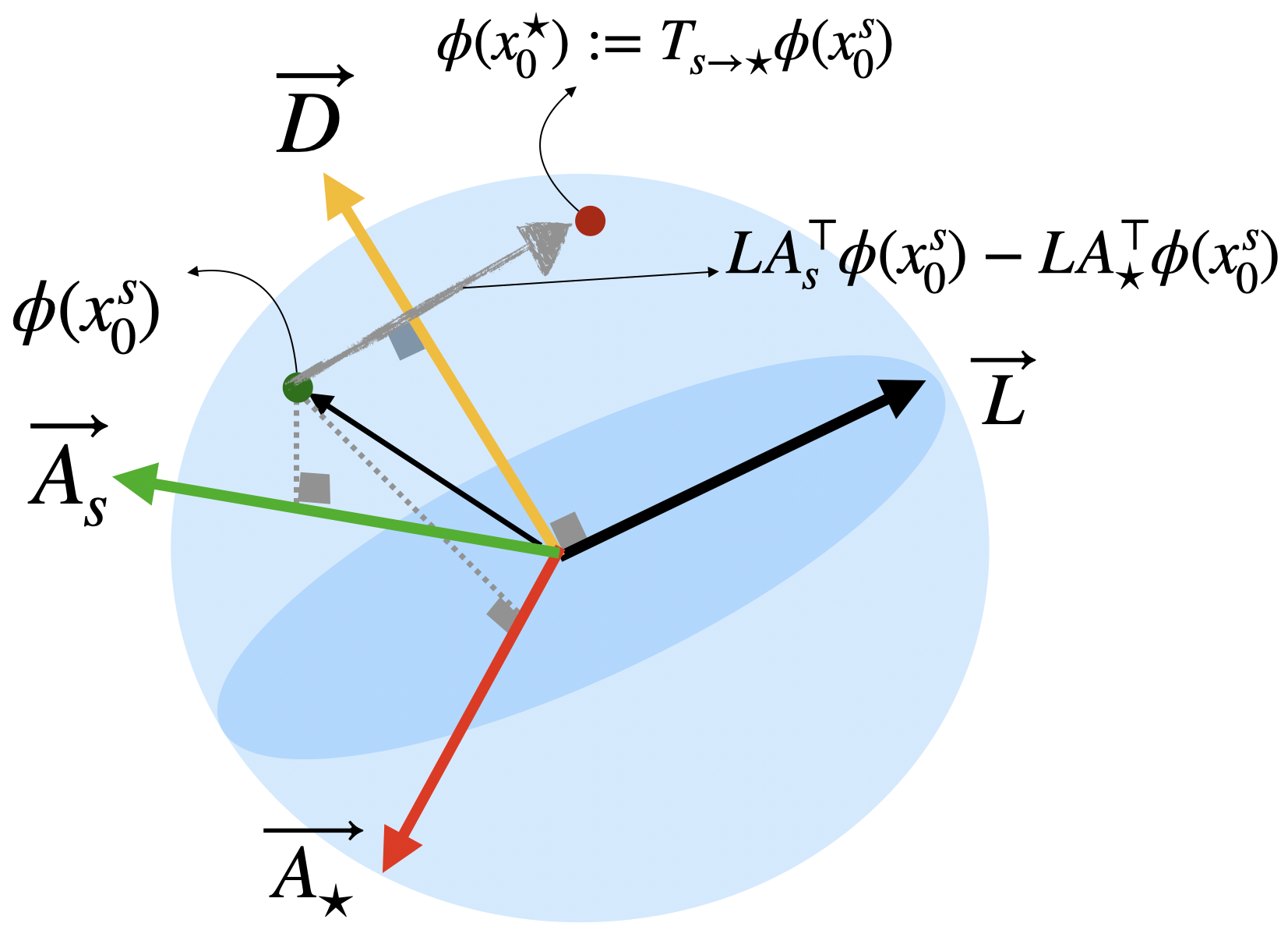}
\caption{Representation transplant. The operator re-points the abduction component of $\phi(x)$ from its source-$s$ reading to a target reading while preserving the deduction component.}
\label{fig:transplant-schematic}
\vspace{-\baselineskip}
\end{wrapfigure}

The operator $T_{s\to s'}$ swaps the abduction component of a representation from what it is in domain $s$ to what it is in domain $s'$, leaving the deduction component fixed; \Cref{sec:related-work} discusses the relationship to linear-intervention operators on neural representations.

The transplant-basis conditions \Cref{eq:transplant-basis} imply, for every $W\in\R^d$ and every $s,s'\in\{1,\dots,K,\star\}$,
\begin{align}
    A_{s'}^\top \,T_{s\to s'}W &= A_s^\top W, \label{eq:transplant-A-property}\\
    D^\top \,T_{s\to s'}W &= D^\top W. \label{eq:transplant-D-property}
\end{align}
The abduction reading at the transplanted vector matches the source-$s$ reading at the original vector, while the deduction reading is preserved (\Cref{lem:multi-bio} in \Cref{app:multi-bio}).

\paragraph{Prediction-invariance on transplanted inputs.} Take a point $x_0^s \in \supp_{P^s}(X)$ with representation $w_0^s := \phi(x_0^s)$, set $w_0^{s'} := T_{s\to s'}\,w_0^s$, and let $x_0^{s'} \in \phi^{-1}(w_0^{s'})$ be a point whose representation matches. Assuming $x_0^{s'} \in \supp_{P^{s'}}(X)$, the transplant-basis conditions \Cref{eq:transplant-basis} and the parameterizations \Cref{eq:abduction-param,eq:deduction-param} imply
\begin{align}
    \vec{P}^{s'}(y_1\mid x_0^{s'},\bm{\tilde u}) &\approx \sigmoid(D^\top w_0^{s'}) = \sigmoid(D^\top\phi(x_0^s)) \approx \vec{P}^s(y_1\mid x_0^s,\bm{\tilde u}), \label{eq:properties-of-w-s-0-D}\\
    \vec{P}^{s'}(\bm{\tilde u}\mid x_0^{s'}) &\approx \softmax(A_{s'}^\top w_0^{s'}) = \softmax(A_s^\top\phi(x_0^s)) \approx \vec{P}^s(\bm{\tilde u}\mid x_0^s). \label{eq:properties-of-w-s-0-A}
\end{align}
By the abduction--deduction entanglement \Cref{eq:AD-entanglement} applied in each domain,
\begin{align*}
    P^s(y_1\mid x_0^s)
    &= \vec{P}^s(\bm{\tilde u}\mid x_0^s)^\top \vec{P}^s(y_1\mid x_0^s,\bm{\tilde u}) \\
    &= \vec{P}^{s'}(\bm{\tilde u}\mid x_0^{s'})^\top \vec{P}^{s'}(y_1\mid x_0^{s'},\bm{\tilde u}) = P^{s'}(y_1\mid x_0^{s'}).
\end{align*}
In words, transplanted points have invariant predictions across the domains they connect. The derivation requires the transplanted representation to admit a same-domain pre-image---a positivity-style condition we state formally as \Cref{ass:closure-multi} in \Cref{sec:algorithm}:
\begin{equation}\label{eq:closure-transport-representations}
    \forall w \in \R^d:\;\; P^s\big(\phi(X) = w\big) > 0 \;\;\iff\;\; P^{s'}\big(\phi(X) = T_{s\to s'}\,w\big) > 0.
\end{equation}
The condition is enforceable across source pairs via a penalty, but for $s'=\star$ it is assumed (target $X$ is unobserved).

\subsection{Invariant representations as a structural consequence}
\label{sec:invariant-rep}

The transplant basis $L$ induces a decomposition of the representation into an abduction component (along $L$) and a deduction component (orthogonal to $L$). Pick any $A\in\R^{d\times r}$ with $A^\top L = I_r$ (any $A_s$ from \Cref{eq:transplant-basis} works). Then $LA^\top$ is an oblique projector onto $\range(L)$, and
\begin{equation}\label{eq:phi-decomp}
    \phi(X) \;=\; \underbrace{LA^\top\phi(X)}_{W_A:\ \text{abduction component}} \;+\; \underbrace{(I_d - LA^\top)\phi(X)}_{W_D:\ \text{deduction component}}.
\end{equation}
The closure condition \Cref{eq:closure-transport-representations} gives matched supports for $P^s(W_A\mid W_D)$ and $P^{s'}(L(A_{s'}^\top-A_s^\top)W_A\mid W_D)$. Under a slightly stronger \emph{distributional-match} condition---namely $P^s(W_A\mid W_D) = P^{s'}\big(W_A\mid W_D\circ T_{s\to s'}\big)$ across domains---the deduction component is an invariant representation. The derivation walks the law of total probability through a cross-domain identity and a fiberwise change of variables:
\begin{align*}
   P^s(y_1\mid w_D)
   &= \sum_{w_A} P^s(y_1\mid W = w_D + w_A)\,P^s(w_A\mid w_D) \tag{total probability} \\
   &= \sum_{w_A} P^{s'}\big(y_1\mid W = T_{s\to s'}(w_D + w_A)\big)\,P^s(w_A\mid w_D) \tag{cross-domain identity} \\
   &= \sum_{w_A'} P^{s'}\big(y_1\mid W = w_D + w_A'\big)\,P^{s'}(w_A'\mid w_D) \tag{relabeling on the fiber} \\
   &= P^{s'}(y_1\mid w_D).
\end{align*}
The cross-domain identity is $P^s(y_1\mid W=w)=P^{s'}(y_1\mid W=T_{s\to s'}w)$ for $w$ in the source-$s$ support whose transplant lands in the source-$s'$ support, which follows from \Cref{eq:abduction-param,eq:deduction-param} and the readout identities \Cref{eq:transplant-A-property,eq:transplant-D-property} via $\softmax(A_s^\top w)^\top\sigmoid(D^\top w)=\softmax(A_{s'}^\top T_{s\to s'}w)^\top\sigmoid(D^\top T_{s\to s'}w)$ (see \Cref{lem:multi-bio} in the appendix). The relabeling step substitutes $w_A':=\pi_A T_{s\to s'}(w_D+w_A)$ with $\pi_A := LA^\top$ the projector onto the abduction component; under \Cref{ass:finite-support} this is a bijection of the (finite) $W_A$-fiber at fixed $W_D=w_D$, and the mass relabeling matches $P^{s'}(w_A'\mid w_D)$ by the distributional-match condition (see \Cref{rem:T-preserves-WD}). The formal proof is in \Cref{app:proofs-cor-invariant-rep}.
This calculation promotes to a corollary of the parameterization:

\begin{corollary}[Invariant representation as a structural consequence]\label{cor:invariant-rep}
Under the parameterization \Cref{eq:abduction-param,eq:deduction-param} with shared deduction \Cref{eq:deduction-invariance}, the transplant-basis conditions \Cref{eq:transplant-basis}, the closure condition \Cref{eq:closure-transport-representations}, and the distributional-match condition $P^s(W_A\mid W_D) = P^{s'}(W_A\mid W_D\circ T_{s\to s'})$ across domains, the deduction component $W_D := (I_d - LA^\top)\phi(X)$ (with $A$ any matrix satisfying $A^\top L = I_r$, e.g.\ any $A_s$) is an invariant representation:
\begin{equation*}
    P^s(Y\mid W_D) \;=\; P^{s'}(Y\mid W_D), \qquad \forall s,s'\in\{1,\dots,K,\star\}.
\end{equation*}
\end{corollary}

Existing work on invariance learning~\citep{arjovsky2019invariant} can be interpreted as attempts to directly recover this deduction component. We note, however, that the deduction component alone is non-identifiable and insufficient for prediction without re-coupling to an abduction; the invariance property here is a structural consequence of the abduction--deduction parameterization and closure under transplant, not an independent learning principle. \Cref{cor:invariant-rep} is not used by the algorithm in \Cref{sec:algorithm}; it situates our parameterization relative to invariance-based DG. The next subsection turns the same machinery in the opposite direction, using the transplant to \emph{generate} candidate target distributions rather than to extract an invariant representation.

\subsection{Plausible-target family and surrogate risk}
\label{sec:plausible-target}

The ground-truth parameters $A_1,\dots,A_K,A_\star,D,L$ are not identified from observational source data alone. We treat them as parameters: tuples $(\{\tilde A_s\},\tilde L)$ with $\tilde A_s^\top \tilde L = I_r$ for every $s\in\{1,\dots,K,\star\}$ and $D^\top \tilde L = 0$, together with parameterized transplants $\tilde T_{s\to s'}$ via \Cref{eq:rep-transport-map}, produce a family of plausible target distributions, each consistent with the source-trained heads. Specifically, applying the same identity used in \Cref{eq:properties-of-w-s-0-D,eq:properties-of-w-s-0-A} to the target,
\begin{equation}\label{eq:w-star-0}
    P^\star(y_1\mid X = x_0^\star) = P^s\big(y_1\mid X\in\phi^{-1}(T_{\star\to s}\phi(x_0^\star))\big),\qquad\forall s\in\{1,\dots,K\},
\end{equation}
which requires \Cref{eq:closure-transport-representations} extended to the target. \Cref{eq:w-star-0} reads target-to-source: it ties a target prediction to a source prediction at the pre-image under $T_{\star\to s}$. The construction below uses the inverse direction, transplanting source representations \emph{into} a target distribution; the two are related by $T_{s\to\star}=T_{\star\to s}^{-1}$ (unipotence; \Cref{lem:multi-bio}).

\paragraph{Closure penalty.} Closure across source pairs is enforceable from data. We penalize it via the Hausdorff distance between actual and transplanted source supports:
\begin{equation}\label{eq:match}
    \mathrm{Closure}(\{\tilde A_s\},\tilde L) := \sum_{s,s'\in\{1,\dots,K\}} \max_{x\in\supp_{P^s}(X)} \min_{x'\in\supp_{P^{s'}}(X)} \big\|\phi(x') - \tilde T_{s\to s'}\phi(x)\big\|.
\end{equation}
Closure to the target regime cannot be enforced empirically because target $X$ is unobserved; it must be assumed.

\paragraph{Fit constraint.} For a candidate target head $\tilde g^\star:\R^d\to[0,1]$ on top of the frozen $\phi$, \Cref{eq:w-star-0} ties source supports to target supports: $\tilde g^\star$ is consistent with $(\{\tilde A_s\},\tilde L)$ if it minimizes cross-entropy on transformed source data,
\begin{equation}\label{eq:loss-grand}
    \mathrm{Fit}(\tilde g^\star,\{\tilde A_s\},\tilde L) := \sum_{s\in\{1,\dots,K\}} \E_{X,Y\sim P^s}\big[\mathrm{CE}\big(\tilde g^\star(\tilde T_{s\to\star}\phi(X)), Y\big)\big].
\end{equation}

\paragraph{Surrogate target distribution and risk.} In standard DG the learner has no target data---not even unlabeled $X$ samples from $P^\star(X)$---so direct evaluation of target risk is unavailable. The transplant identity \Cref{eq:w-star-0} ties source supports to target supports, motivating evaluation on a surrogate $\tilde P^\star(X,Y) = \tilde P^\star(X) \cdot P^\star(Y\mid X)$ with $\supp_{\tilde P^\star}(X) = \supp_{P^\star}(X)$. With discrete support and a complex-enough hypothesis class, $\argmin_{h:\cX\to\cY} R_{\tilde P^\star}(h) = \argmin_{h:\cX\to\cY} R_{P^\star}(h)$, so minimizing surrogate risk is a valid proxy for minimizing target risk. A practical choice for $\tilde P^\star(X)$ is to apply transplant to source representations:
\begin{align}
    \tilde P^\star(\phi(x), y)
    &:= \tilde P^\star(\phi(x))\cdot P^\star(y_1\mid \phi(x)) \nonumber\\
    &= P^s(\tilde T_{s\to\star}\phi(X))\cdot \tilde g^\star(\tilde T_{s\to\star}\phi(X)). \label{eq:surrogate-distribution}
\end{align}
This is consistent with \Cref{eq:loss-grand}. The \emph{surrogate target risk} of a classifier $h\in\cH$ is then
\begin{equation}\label{eq:risk-grand}
    \mathrm{Risk}(h;\tilde g^\star,\{\tilde A_s\},\tilde L) := \E_{W\sim\tilde P^\star(\phi(X)),\,Y\sim\tilde g^\star(W)}\big[\ell(h(W), Y)\big].
\end{equation}
The two heads play distinct roles: $\tilde g^\star$ is the \emph{adversary's} candidate target head---part of the transplant parameterization $(\tilde g^\star,\{\tilde A_s\},\tilde L)$ that synthesizes a plausible target $(X,Y)$ distribution---while $h\in\cH$ is the \emph{learner's} classifier evaluated against that synthesized distribution. The learner's objective therefore depends on the adversary's choice of head through $\tilde P^\star$.

\paragraph{Best worst-case objective.} We upper-bound the surrogate target risk of a classifier $h$ by maximizing the risk while encouraging $\mathrm{Fit}$ optimality and penalizing $\mathrm{Closure}$:
\begin{equation}\label{eq:obj}
    \mathrm{Obj}\big(\tilde g^\star,\{\tilde A_s\},\tilde L; h\big) := \mathrm{Risk}(h;\tilde g^\star,\{\tilde A_s\},\tilde L) - \lambda_1\cdot\mathrm{Fit}(\tilde g^\star,\{\tilde A_s\},\tilde L) - \lambda_2\cdot\mathrm{Closure}(\{\tilde A_s\},\tilde L),
\end{equation}
with $\lambda_1,\lambda_2$ regularization hyperparameters controlling how strictly the theoretical constraints are imposed. We seek $h\in\cH$ minimizing the surrogate-risk upper bound:
\begin{equation}\label{eq:minimax}
    h^\star \;\in\; \argmin_{h\in\cH}\;\mathrm{Risk}\big(h;\,\argmax_{\tilde g^\star,\{\tilde A_s\},\tilde L}\mathrm{Obj}(\tilde g^\star,\{\tilde A_s\},\tilde L; h)\big).
\end{equation}
\Cref{sec:algorithm} presents an algorithm that solves this min-max optimization and a guarantee on its solution.

\section{Algorithm and Guarantee}
\label{sec:algorithm}

\subsection{A learner--adversary game}
\label{sec:learner-adversary}

\begin{algorithm}[t]
\caption{Causal Robust Optimization (CRO)}
\label{alg:CRO}
\begin{algorithmic}[1]
\small
\REQUIRE Empirical sources $\hat P^1,\dots,\hat P^K$; representation $\phi$; hypothesis class $\cH$; hyperparameters $\lambda_1,\lambda_2$; tolerance $\epsilon>0$.
\ENSURE Classifier $h^\star \in \cH$.
\STATE $\hat\cP^\star \gets \emptyset$;\; $h_0 \gets$ ERM on pooled $\hat P^1,\dots,\hat P^K$;\; $t\gets 0$;\; $M_0\gets-\infty$.
\REPEAT
    \STATE $t\gets t+1$.
    \STATE $\big(\hat g^\star_t,\{\hat A^{(t)}_s\},\hat L_t\big) \in \argmax_{\tilde g^\star,\{\tilde A_s\},\tilde L}\widehat{\mathrm{Obj}}(\tilde g^\star,\{\tilde A_s\},\tilde L; h_{t-1})$.
    \STATE Form $\hat P^\star_t$ via \Cref{eq:surrogate-distribution};\; $\hat\cP^\star \gets \hat\cP^\star\cup\{\hat P^\star_t\}$.
    \STATE Learner: $h_t \in \argmin_{h\in\cH}\max_{\hat P\in\hat\cP^\star} R_{\hat P}(h)$.
    \STATE $M_t \gets \max_{\hat P\in\hat\cP^\star} R_{\hat P}(h_t)$.
\UNTIL $t\geq 2$ \textbf{and} $|M_t - M_{t-1}|\leq\epsilon$.
\RETURN $h_t$.
\end{algorithmic}
\end{algorithm}

The transplant parameterization of \Cref{sec:transplant} yields, for any $h$, an upper bound on the surrogate target risk by searching over abduction--deduction pairs consistent with the source distributions. To find $h^\star$ minimizing this bound (\Cref{eq:minimax}), we adapt the Causal Robust Optimization (CRO) algorithm of~\citet{jalaldoust2024partial}: a game between a learner and an adversary in which, at iteration $t$, the learner proposes a candidate $h_t \in \cH$ and the adversary uses the transplant parameterization to find a \emph{hard} plausible target distribution $\hat P^\star_t$ on which $h_t$'s risk is large. The learner adds $\hat P^\star_t$ to a growing collection $\hat\cP^\star = \{\hat P^\star_1,\dots,\hat P^\star_t\}$ and updates $h_t$ to minimize the maximum risk over the collection. The game terminates when this maximum risk converges. In practice we work with empirical distributions $\hat P^1,\dots,\hat P^K$ and write $\widehat{\mathrm{Obj}}$ for the plug-in estimator of $\mathrm{Obj}$. \Cref{alg:CRO} gives the pseudocode.

The adversary step (line 4), at fixed $h_{t-1}$, bounds the surrogate target risk by maximizing over transplant parameters subject to the source-fit term $\mathrm{Fit}$ (\Cref{eq:loss-grand}) and the closure penalty $\mathrm{Closure}$ (\Cref{eq:match}). The learner step (line 6) refines $h$ to make this bound tight. The two steps together implement the min-max in \Cref{eq:minimax}.

\paragraph{Frozen vs.\ refined predictors.} If $h$ is frozen at the initial pooled-source ERM and the outer learner step is omitted, CRO becomes a pure partial evaluation of target risk for that fixed predictor. Including the learner step refines $h$ to minimize that bound, yielding the minimax guarantee below. The two regimes share the transplant-parameterized adversary; they differ only in whether $h$ is fixed or updated.

\subsection{Minimax optimality}
\label{sec:minimax}

The guarantee for CRO depends on the adversary's search space matching the family of plausible target distributions. We collect the five structural and regularity conditions used in the theorem.

\begin{assumption}[Confounded labels, multi-domain]\label{ass:confounded-multi}
The source and target domains induce the causal graph in \Cref{fig:SD-for-bow-graph}.
\end{assumption}

\begin{assumption}[Finite support]\label{ass:finite-support}
$Y$ is binary and $X$ has finite support.
\end{assumption}

\begin{assumption}[Invariant deduction]\label{ass:invariant-deduction}
There exists a coarsening $\tilde U \subset U$ such that $\vec{P}^s(y_1\mid x,\bm{\tilde u}) = \vec{P}^{s'}(y_1\mid x,\bm{\tilde u})$ for every $x\in\cX$ and $s,s'\in\{1,\dots,K,\star\}$.
\end{assumption}

\begin{remark}[When $f_Y$-invariance gives \Cref{ass:invariant-deduction}]\label{rem:fY-sufficiency}
Invariance of the structural label mechanism $f_Y$---i.e., $P^s(Y\mid X,U)=P^{s'}(Y\mid X,U)$---implies \Cref{ass:invariant-deduction} only when the coarsening is \emph{sufficient for $Y$ given $X$}: $P(Y\mid X,U)=P(Y\mid X,\tilde U)$, equivalently $Y\perp U\mid X,\tilde U$. Without this sufficiency the within-fiber marginalization $P^s(Y\mid X,\tilde U)=\sum_{u:\psi(u)=\tilde u}P(Y\mid X,U=u)\,P^s(U\mid X,\tilde U)$ inherits the domain-varying $\pi_s$ through $P^s(U\mid X,\tilde U)$, so $f_Y$-invariance and deduction invariance can come apart. \Cref{ass:invariant-deduction} is stated directly on the coarsening to sidestep this subtlety.
\end{remark}

\begin{assumption}[Linear abduction--deduction, multi-domain]\label{ass:linear-multi}
There exist $A_1,\dots,A_K,A_\star\in\R^{d\times r}$ (per-domain abduction parameters) and a single shared $D\in\R^{d\times r}$ (deduction parameter) such that for every $s\in\{1,\dots,K,\star\}$ and every $x\in\cX$,
\begin{equation*}
\vec{P}^s(\bm{\tilde u}\mid x) = \softmax(A_s^\top\phi(x)),
\qquad
\vec{P}^s(y_1\mid x,\bm{\tilde u}) = \sigmoid(D^\top\phi(x)),
\end{equation*}
both read as $r$-vectors indexed by $\tilde u\in\tilde\cU$ in a fixed ordering, with $\sigmoid$ applied entry-wise.
\end{assumption}

\begin{assumption}[Closure under transplant, multi-domain]\label{ass:closure-multi}
There exists $L\in\R^{d\times r}$ with $A_s^\top L = I_r$ for every $s\in\{1,\dots,K,\star\}$ and $D^\top L = 0$, such that for every $s,s'\in\{1,\dots,K,\star\}$,
\begin{equation*}
    \forall w\in\R^d:\;\; P^s\big(\phi(X) = w\big) > 0 \;\iff\; P^{s'}\big(\phi(X) = T_{s\to s'}\,w\big) > 0.
\end{equation*}
\end{assumption}

\begin{theorem}[Minimax-optimality of CRO over the plausible target set]\label{thm:multi-id}
    Suppose access to infinite source data, i.e., perfect samplers for $P^1(X,Y),\dots,P^K(X,Y)$, and that each gradient-based search in CRO (\Cref{alg:CRO}) attains its global optimum. Under \Cref{ass:confounded-multi,ass:finite-support,ass:invariant-deduction,ass:linear-multi,ass:closure-multi}, CRO terminates in finitely many iterations for every tolerance $\epsilon>0$, and the worst-case target risk of the returned classifier $h^\star \in \cH$---measured over the plausible target set entailed by the structural assumptions---converges to the minimax value as $\epsilon\to 0$:
\begin{equation*}
    \max_{\tilde P^\star\in\cP^\star(P^1,\dots,P^K)} R_{\tilde P^\star}(h^\star) \;\xrightarrow{\epsilon\to 0}\; \min_{h\in\cH}\;\max_{\tilde P^\star\in\cP^\star(P^1,\dots,P^K)} R_{\tilde P^\star}(h),
\end{equation*}
    where $\cP^\star(P^1,\dots,P^K)$ is the set of target distributions entailed by any $(K+1)$-tuple of SCMs compatible with \Cref{ass:confounded-multi,ass:finite-support,ass:invariant-deduction,ass:linear-multi,ass:closure-multi}.
\end{theorem}

The proof, deferred to \Cref{app:proofs-multi-id}, follows from the partial-transportability identity \Cref{eq:w-star-0} together with the convergence argument adapted from~\citet{jalaldoust2024partial}. The statement is best read as a conditional guarantee: \emph{if} the structural assumptions \Cref{ass:confounded-multi,ass:finite-support,ass:invariant-deduction,ass:linear-multi,ass:closure-multi} hold, CRO is minimax-optimal over the parameterized plausible set; the guarantee does not claim DG optimality outside this set.

\begin{remark}[Target-side closure is not empirically verifiable]\label{rem:target-closure-caveat}
The target component of \Cref{ass:closure-multi} (at $s'=\star$) is structural, not statistical: target $X$ is unobserved, so the empirical Hausdorff penalty (\Cref{eq:match}) enforces closure only across source pairs. The conclusion of \Cref{thm:multi-id} is therefore conditional on target-side closure as an assumption rather than as a tested condition; we return to this caveat in the limitations.
\end{remark}

\paragraph{Role of the assumptions.} The five assumptions divide as follows. \Cref{ass:confounded-multi,ass:invariant-deduction} are structural: the same confounded $X\to Y$ mechanism across domains, with deduction held invariant. \Cref{ass:finite-support} is a technical regularity. \Cref{ass:linear-multi} is the parameterization choice---softmax-linear abduction and logit-linear (shared-$D$) deduction in the pretrained representation space, which is what makes a small set of $(A_s, D, L)$ parameters expressive. \Cref{ass:closure-multi} is the load-bearing positivity-style condition: every transplant has a same-domain pre-image, including across the (unseen) target. The penalty in \Cref{eq:match} enforces a finite-sample analog of closure across source pairs only; closure to the target regime is assumed.

\subsection{Robustness and finite-sample considerations}
\label{sec:robustness}

\Cref{thm:multi-id} is an asymptotic statement over a structurally idealized set: infinite source data, exact closure, and a global oracle for the inner searches. Two natural concerns are how the guarantee degrades when (i) closure holds only approximately---the unverifiable target-side condition of \Cref{rem:target-closure-caveat}---and (ii) source data is finite. We address each with a short statement; full proofs are in \Cref{app:robustness}.

\begin{definition}[$\epsilon$-relaxed closure]\label{def:eps-closure}
For $\epsilon\ge 0$, the $\epsilon$-relaxed version of \Cref{ass:closure-multi} requires the following: for every $s,s'\in\{1,\dots,K,\star\}$ and every $w\in\R^d$ with $P^s(\phi(X)=w)>0$, there exists $w'\in\R^d$ with $P^{s'}(\phi(X)=w')>0$ and $\|w'-T_{s\to s'}w\|\le\epsilon$, and symmetrically with $s,s'$ swapped. Exact closure is the case $\epsilon=0$.
\end{definition}

Let $\cP^\star_\epsilon$ denote the plausible target set under $\epsilon$-relaxed closure (\Cref{ass:closure-multi} replaced by \Cref{def:eps-closure}), and let $V^\star_\epsilon:=\min_{h\in\cH}\max_{\tilde P^\star\in\cP^\star_\epsilon}R_{\tilde P^\star}(h)$ be its minimax value. Since the relaxed condition admits a strict superset of distributions ($\cP^\star_0\subseteq\cP^\star_\epsilon$), $V^\star_\epsilon\ge V^\star_0$. The next theorem bounds the gap.

\begin{theorem}[Graceful degradation under approximate closure]\label{thm:robustness}
Suppose \Cref{ass:confounded-multi,ass:finite-support,ass:invariant-deduction,ass:linear-multi} hold, and that $\cH$ consists of $L_h$-Lipschitz heads $h:\R^d\to[\delta,1-\delta]$ for some $\delta\in(0,1/2)$; the loss is cross-entropy. Then
\begin{equation}\label{eq:robustness-bound}
0\;\le\;V^\star_\epsilon - V^\star_0\;\le\;\frac{L_h}{\delta}\cdot\epsilon.
\end{equation}
\end{theorem}

The bound is linear in the closure residual $\epsilon$, so CRO's worst-case target risk degrades smoothly rather than catastrophically as closure fails. The constant $L_h/\delta$ surfaces the two practitioner-facing knobs: head Lipschitzness on the representation space, and a confidence floor that keeps cross-entropy away from its singularity. The proof, in \Cref{app:robustness-proof}, couples each ``relaxed-closure'' transplant target to an exact-closure one within $\ell^2$-distance $\epsilon$ and propagates through the loss.

\paragraph{Finite-sample plug-in error.} The minimax-optimality of \Cref{thm:multi-id} assumes infinite source data; the algorithm in practice uses empirical $\hat P^s$ with $N$ samples per source. A uniform-convergence argument over the (compact) parameter space $\Theta_B$ of Step~3 of the \Cref{thm:multi-id} proof gives the following.

\begin{proposition}[Finite-sample plug-in error, informal]\label{prop:finite-sample}
Let $\hat h$ be the output of \Cref{alg:CRO} run on empirical $\hat P^1,\dots,\hat P^K$ of size $N$ per source, and let $\Theta_B$ be the parameter space of Step~3 of the \Cref{thm:multi-id} proof, with metric entropy $\log\mathcal{N}(\Theta_B,\eta)$ at scale $\eta$. Under \Cref{ass:confounded-multi,ass:finite-support,ass:invariant-deduction,ass:linear-multi,ass:closure-multi}, with probability at least $1-\delta$,
\begin{equation}\label{eq:finite-sample-bound}
\max_{\tilde P^\star\in\cP^\star}R_{\tilde P^\star}(\hat h)
\;\le\;V^\star_0
\;+\;\underbrace{C_1\sqrt{\frac{\log\mathcal{N}(\Theta_B,\eta)+\log(1/\delta)}{N}}}_{\text{statistical error}}
\;+\;\underbrace{\epsilon_{\mathrm{opt}}}_{\text{optimization}}
\;+\;\underbrace{\epsilon_{\mathrm{Haus}}}_{\text{closure residual}}
\;+\;\eta,
\end{equation}
for an absolute constant $C_1$ depending on the cross-entropy Lipschitz constant $1/\delta$ and the operator-norm bound $B$ in \textnormal{(P1)}, where $\epsilon_{\mathrm{opt}}$ absorbs the suboptimality of the inner adversary and learner searches relative to their global optima, and $\epsilon_{\mathrm{Haus}}=\mathrm{Closure}(\hat A_s,\hat L)$ is the residual empirical Hausdorff penalty at the returned parameters.
\end{proposition}

The statement names the three terms a practitioner can read off the algorithm's outputs: a $1/\sqrt N$ statistical rate (with $\log\mathcal{N}(\Theta_B,\eta)$ playing the role of effective complexity), an optimization error governed by the global-optimum gap of the inner searches (zero under the assumption of \Cref{thm:multi-id}, but typically positive in practice), and the source-pair Hausdorff residual that the algorithm minimizes during training. Tight constants for the softmax/sigmoid-linear-in-$\phi$ class via bracketing entropy are standard and orthogonal to the paper's contribution; we defer them to future work.

\section{Related Work}
\label{sec:related-work}

We situate representation transplant and CRO with respect to four lines of prior work: causality-inspired domain generalization, distributional robustness, transportability and partial transportability, and linear interventions on neural representations. The transplant parameterization sits at their intersection: a linear operator on a frozen representation (lineage four) used as a primitive for partial transportability (lineage three) and operationalized via a learner--adversary game for worst-case DG (lineages one and two).

\paragraph{Causality-inspired domain generalization.} Causality-inspired DG methods build on related principles under different assumptions. Invariant prediction~\citep{peters2016causal} establishes structural conditions under which heterogeneous data partially identify the causal parents of $Y$ among covariates $X = (X_1,\dots,X_d)$, leveraging the invariance of source-prediction rules based on causal parents; \citet{rojascarulla2018invariant} extend this principle to causal transfer; IRM~\citep{arjovsky2019invariant} and V-REx~\citep{krueger2021out} operationalize related principles as regularized learning objectives. DomainBed-style empirical evaluations~\citep{gulrajani2021search}, however, have not validated meaningful gains over ERM, and the theoretical claims rely on assumptions---in particular the absence of unobserved confounding of $Y$---that may not hold. Our setting differs by treating the $X\to Y$ mechanism as confounded by $\tilde U$ and parameterizing the resulting non-identifiability via the abduction--deduction factorization (\Cref{sec:transplant}); this is structurally distinct from a parents-of-$Y$ recovery target. \Cref{cor:invariant-rep} recovers an invariant representation as a consequence of our parameterization rather than as a primitive learning goal.

\paragraph{Robustness and worst-case generalization.} Distributional robustness reformulates DG as a min-max over a family of distributions: Group DRO~\citep{sagawa2020distributionally} optimizes worst-case risk over annotated subgroups. CRO (\Cref{alg:CRO}) is a min-max of similar shape, but the family ranged over is structurally derived---the transplant-parameterized set of plausible targets entailed by source data and \Cref{ass:confounded-multi,ass:finite-support,ass:invariant-deduction,ass:linear-multi,ass:closure-multi}---rather than user-specified or annotation-driven. On the evaluation side, WILDS~\citep{koh2021wilds} curates real-world distribution-shift benchmarks that complement the DomainBed protocol~\citep{gulrajani2021search} used here.

\paragraph{Transportability and partial transportability.} The transportability literature~\citep{pearl2011transportability,bareinboim2016transportability} formalizes when causal effects estimated in a source population transfer to a target. \emph{Partial} transportability~\citep{lee2019general,jalaldoust2024partial} addresses settings where transfer is not point-determined but admits informative bounds; the multi-domain construction in \Cref{sec:transplant} is most directly an instance of partial transportability of $P^\star(Y\mid X)$, with the transplant operator providing a concrete parameterization on a pretrained representation.

\paragraph{Linear interventions on neural representations.} Mechanically, representation transplant is a low-rank linear intervention on a frozen pretrained representation, parameterized by $(\{A_s\}, L)$. Linear interventions on internal NN representations have a substantial recent history: interchange interventions for testing causal abstractions~\citep{geiger2021causal} and their extension to interchange intervention training~\citep{geiger2022inducing}; Distributed Alignment Search (DAS)~\citep{geiger2024finding}, which aligns interpretable causal variables with distributed subspaces of a frozen representation; activation patching for localizing factual knowledge~\citep{meng2022locating}; representation finetuning (LoReFT)~\citep{wu2024reft} for parameter-efficient adaptation; activation steering~\citep{turner2023activation} and representation engineering~\citep{zou2023representation} for behavioral control; and concept erasure (INLP, R-LACE, LEACE)~\citep{ravfogel2020null,ravfogel2022adversarial,belrose2023leace} via learned linear projectors. Our construction is mechanically a member of this broader family. To the best of our knowledge, the most similar concept to our representation transplant is DAS~\citep{geiger2024finding}, together with its representation-finetuning extension LoReFT~\citep{wu2024reft}: DAS shares the counterfactual-style interpretation of a low-rank linear operator on a frozen $\phi(X)$---asking what the model would predict if a high-level causal variable were swapped via a representation-level interchange---but is deployed for mechanistic interpretability (verifying that an interpretable causal variable is aligned with a subspace of the representation), whereas our transplant is deployed for partial transportability and minimax DG. We adopt the same operator family for a different purpose---as a parameterization of the abduction--deduction non-identifiability of a trained predictor---and use it as a primitive for partial transportability and minimax DG. To our knowledge, no prior work uses linear-intervention operators for these targets, nor formalizes a closure-under-intervention assumption of the kind in \Cref{ass:closure-multi}.

\section{Experiments}
\label{sec:experiments}

We evaluate Causal Robust Optimization (CRO; \Cref{alg:CRO}) on a designed synthetic benchmark (Colored MNIST) and two real-world benchmarks (PACS, OfficeHome), and close with an empirical pass over the four structural premises that underlie the minimax guarantee. All evaluation follows the DomainBed protocol of~\citet{gulrajani2021search}: leave-one-environment-out cells, three random seeds per cell, identical hyperparameter budgets across methods, and training-domain validation as the model-selection rule. Source-validation accuracy is the only selection rule we report---target labels are never observed during model selection on any benchmark.

\subsection{Implementation details}
\label{sec:experiments-impl}

CRO is implemented on top of the DomainBed framework~\citep{gulrajani2021search} as a new \texttt{Algorithm} subclass; baselines use the official DomainBed implementations. The featurizer $\phi$ is ResNet-50 pretrained on ImageNet with AugMix for the image benchmarks (PACS, OfficeHome) and the standard 4-layer \texttt{MNIST\_CNN} of DomainBed for Colored MNIST; in both cases the final classification head $h$ is a linear layer on top of frozen $\phi$. Training proceeds in two phases: Phase~1 is standard ERM-style training of $(\phi, h)$ on the union of source domains; Phase~2 is the CRO outer loop (\Cref{alg:CRO}) refining $h$ against the transplant-parameterized adversary. Per benchmark we sample 20 hyperparameter configurations $\times$ 3 trial seeds (180 runs on Colored MNIST, 240 on each of PACS, OfficeHome) and select per cell by source-validation accuracy. Full architecture, two-phase training schedule, hyperparameter search ranges and per-dataset compute tiers, dataset statistics, GPU/wall-clock budgets, and reproducibility identifiers are documented in \Cref{app:experiments}.

\subsection{Synthetic data: Colored MNIST}
\label{sec:experiments-colored-mnist}

\paragraph{Task.} Colored MNIST~\citep{arjovsky2019invariant} is a binary classification of MNIST digits ($Y=0$ for digits $0$--$4$, $Y=1$ for digits $5$--$9$) where each image is colorized red or green according to a domain-specific correlation $\rho_s$ with the label. The source environments correspond to $\rho = +0.9$ and $\rho = +0.8$ (color tracks the label); the diagnostic environment is $\rho = -0.9$ (color anti-tracks the label). Color is a spurious feature whose label-correlation flips between source and target; shape is the invariant feature. The benchmark therefore isolates whether a method picks up the invariant structure when it is in tension with the spurious one---a designed instance of the shortcut-learning phenomenon documented across real-world DG benchmarks~\citep{geirhos2020shortcut,beery2018recognition}.

\paragraph{Source vs.\ target risk.}
\begin{figure}[t]
\centering
\includegraphics[width=0.45\linewidth]{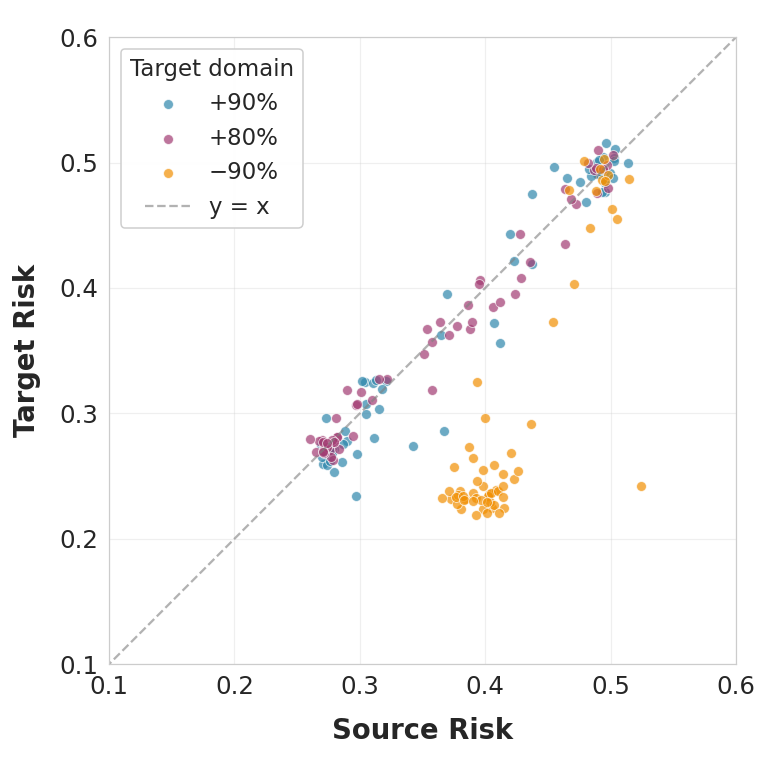}
\caption{Source vs.\ target risk across CRO training checkpoints on Colored MNIST, colored by held-out target.}
\label{fig:source-vs-test}
\end{figure}
Before reporting benchmark numbers, it is informative to look at how source-validation risk relates to held-out target risk across CRO checkpoints. \Cref{fig:source-vs-test} plots this relationship across the three Colored MNIST leave-one-out cells. The $+90\%$ and $+80\%$ cells (blue, magenta) lie close to $y=x$: source-validation is a tight proxy for target performance. The $-90\%$ cell (orange) is the contrast---at source risk $\approx 0.35$--$0.42$, many checkpoints achieve target risk $\approx 0.22$--$0.30$, far below where the diagonal would put them, and the spread at fixed source risk is wide. The diagnostic cell is therefore exactly where source-validation is a poor proxy for held-out generalization, and is where standard ERM-style baselines that select by source-validation collapse in \Cref{tab:cmnist}.

\paragraph{Results.} \Cref{tab:cmnist} reports leave-one-environment-out test accuracy under training-domain validation. Baseline numbers are reproduced from the DomainBed reference table~\citep{gulrajani2021search}; CRO numbers are our runs averaged over three seeds.

\begin{table}[t]
\caption{Colored MNIST: leave-one-environment-out test accuracy under DomainBed-style training-domain validation. Baseline numbers from~\citet{gulrajani2021search} under the same featurizer (DomainBed's \texttt{MNIST\_CNN}), the same training-domain validation rule, and the same 20-configuration $\times$ 3-seed hyperparameter budget as CRO; CRO is our runs averaged over three seeds. Best per column in bold.}
\label{tab:cmnist}
\centering
\small
\begin{tabular}{lcccc}
\toprule
\textbf{Algorithm}   & \textbf{$+$90\%}     & \textbf{$+$80\%}     & \textbf{$-$90\%}     & \textbf{Avg}         \\
\midrule
ERM                  & 71.7 $\pm$ 0.1       & 72.9 $\pm$ 0.2       & 10.0 $\pm$ 0.1       & 51.5                 \\
IRM                  & 72.5 $\pm$ 0.1       & 73.3 $\pm$ 0.5       & 10.2 $\pm$ 0.3       & 52.0                 \\
GroupDRO             & 73.1 $\pm$ 0.3       & 73.2 $\pm$ 0.2       & 10.0 $\pm$ 0.2       & 52.1                 \\
Mixup                & 72.7 $\pm$ 0.4       & 73.4 $\pm$ 0.1       & 10.1 $\pm$ 0.1       & 52.1                 \\
MLDG                 & 71.5 $\pm$ 0.2       & 73.1 $\pm$ 0.2       & 9.8 $\pm$ 0.1        & 51.5                 \\
CORAL                & 71.6 $\pm$ 0.3       & 73.1 $\pm$ 0.1       & 9.9 $\pm$ 0.1        & 51.5                 \\
MMD                  & 71.4 $\pm$ 0.3       & 73.1 $\pm$ 0.2       & 9.9 $\pm$ 0.3        & 51.5                 \\
DANN                 & 71.4 $\pm$ 0.9       & 73.1 $\pm$ 0.1       & 10.0 $\pm$ 0.0       & 51.5                 \\
CDANN                & 72.0 $\pm$ 0.2       & 73.0 $\pm$ 0.2       & 10.2 $\pm$ 0.1       & 51.7                 \\
MTL                  & 70.9 $\pm$ 0.2       & 72.8 $\pm$ 0.3       & 10.5 $\pm$ 0.1       & 51.4                 \\
SagNet               & 71.8 $\pm$ 0.2       & 73.0 $\pm$ 0.2       & 10.3 $\pm$ 0.0       & 51.7                 \\
ARM                  & \textbf{82.0} $\pm$ 0.5 & \textbf{76.5} $\pm$ 0.3 & 10.2 $\pm$ 0.0       & 56.2                 \\
VREx                 & 72.4 $\pm$ 0.3       & 72.9 $\pm$ 0.4       & 10.2 $\pm$ 0.0       & 51.8                 \\
RSC                  & 71.9 $\pm$ 0.3       & 73.1 $\pm$ 0.2       & 10.0 $\pm$ 0.2       & 51.7                 \\
\textbf{CRO (ours)}  & 72.1 $\pm$ 0.7       & 72.4 $\pm$ 0.3       & \textbf{76.0} $\pm$ 0.7 & \textbf{73.5}      \\
\bottomrule
\end{tabular}
\end{table}

\paragraph{Discussion.} The $-90\%$ column is the diagnostic. All prior baselines collapse to $\approx 10\%$ accuracy (worse than chance): they inherit the source's color--label rule and apply it to a target where that rule is inverted, so source-validation selects a predictor that is systematically wrong on the held-out cell. CRO at $76.0\%$ indicates the transplant-parameterized adversary surfaces an abduction--deduction factorization in which color sits in the abduction subspace (the $L$-direction) and shape in the deduction subspace, with the learner converging onto the deduction readout. On the two source-tracking cells CRO matches the baseline cluster ($72.1\%/72.4\%$); the per-column wins of ARM ($82.0\%/76.5\%$) on those cells reflect its ability to absorb the source-side color correlation, which is exactly the rule that gets inverted on the diagnostic cell.

\begin{figure}[t]
\centering
\includegraphics[width=0.55\textwidth]{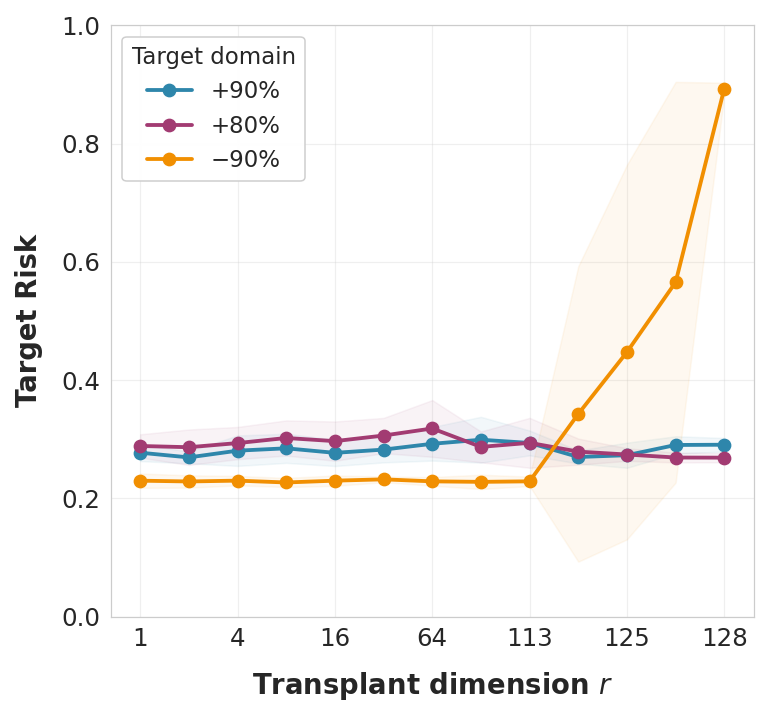}
\caption{CRO target risk on Colored MNIST as a function of the transplant dimension $r$. The diagnostic $-90\%$ cell is flat-low through $r\leq 113$ and collapses by $r\in\{125,128\}$ (transition in $r\in(113,125)$); the source-tracking cells are insensitive to $r$.}
\label{fig:rank-vs-test}
\end{figure}

\paragraph{Stress test: transplant dimension.}
The transplant dimension $r$ is the size of the parameterized coarsening of $\tilde U$ (\Cref{eq:abduction-param,eq:deduction-param}). The true support size $|\tilde\cU|$ is generally unknown, so we stress-test sensitivity to $r$ on Colored MNIST: \Cref{fig:rank-vs-test} sweeps $r$ from $1$ to the feature dimension $128$ and reports CRO target risk on all three leave-one-out cells. The two source-tracking cells ($+90\%$, $+80\%$) are flat in $r$ at target risk $\approx 0.28$--$0.30$. The diagnostic $-90\%$ cell is flat-low ($\approx 0.23$) for $r\leq 113$ and then climbs sharply, reaching target risk $\approx 0.9$ at $r=128$. The collapse beyond $r\approx 113$ happens because the orthogonal subspace $L^\perp$---the deduction component of $\phi$---becomes too small to carry an accurate invariant deduction map, so the worst-case risk on the diagnostic target explodes. The other two held-out cells do not collapse: at large $r$, even with a poor deduction component the abduction component can compensate to fit the source predictions, and on a target whose abduction is well-aligned with the sources, the ensemble prediction remains good. Small to moderate $r$ avoids both failure modes and is the operating point we use for the headline numbers in \Cref{tab:cmnist}.

\paragraph{CRO termination.}
\begin{figure}[t]
\centering
\includegraphics[width=0.55 \linewidth]{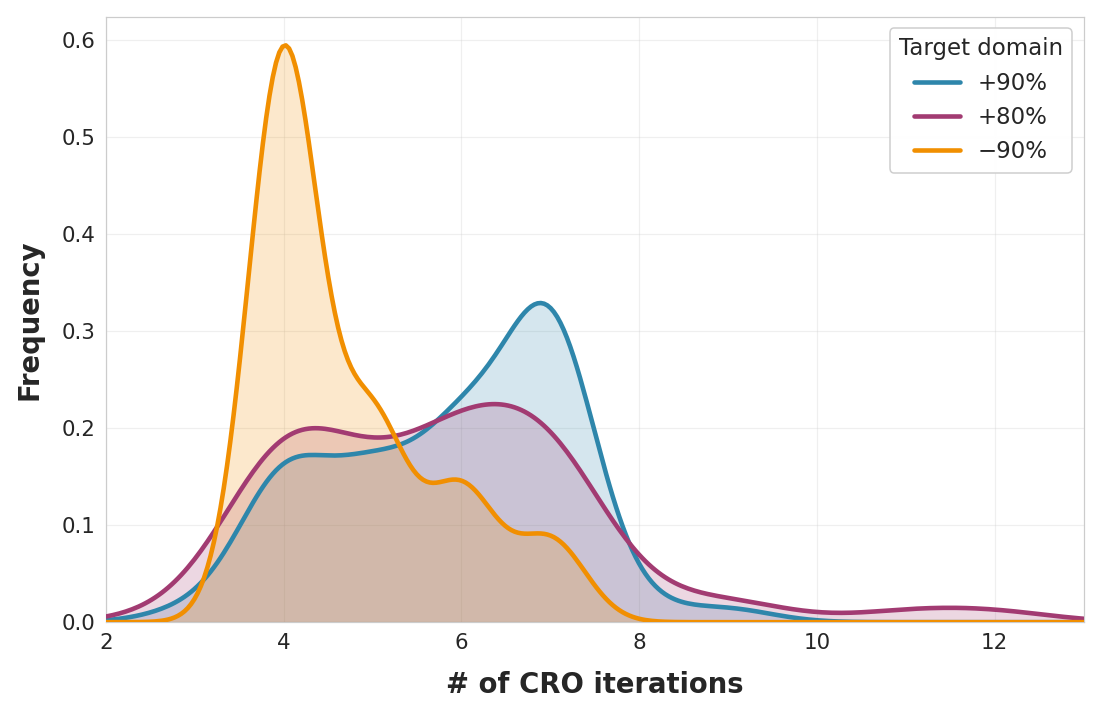}
\caption{Distribution of CRO outer-loop iterations before termination on Colored MNIST, across seeds and hyperparameter configurations, by held-out target.}
\label{fig:cmnist-iter-density}
\end{figure}
CRO (\Cref{alg:CRO}) terminates when the worst-case risk over the growing collection of plausible targets stops increasing materially. \Cref{fig:cmnist-iter-density} reports the empirical distribution of iteration counts at termination across seeds and hyperparameter configurations. The diagnostic $-90\%$ cell terminates fastest (mode $\approx 4$, almost all mass between $3$ and $6$): the adversary identifies a small number of worst-case targets very quickly, and the learner converges to a worst-case-optimal predictor without needing many rounds. The source-tracking cells take a few more rounds (mode $\approx 7$ for $+90\%$, bimodal at $4$ and $6$ for $+80\%$), consistent with a less peaked worst-case landscape. Across the full sweep, termination occurs in at most $13$ outer iterations, so CRO's compute budget is driven by the inner adversary and learner updates rather than the outer loop count.

\subsection{Real-world datasets: PACS and OfficeHome}
\label{sec:experiments-realworld}

The two real-world benchmarks test CRO in regimes where the deduction-invariance assumption is at best approximate---rendering style and category-level appearance can themselves vary across domains, partially shifting the deduction map. We report numbers under the same DomainBed training-domain-validation protocol and analyze where CRO matches, leads, or trails strong baselines.

\paragraph{PACS.} PACS~\citep{li2017deeper} is a four-domain object-recognition benchmark (photo, art-painting, cartoon, sketch) with seven classes. Spurious features are texture and rendering style; the invariant features are object shape and category-level structure. Each cell holds out one of the four domains and trains on the other three.

\begin{table}[t]
\caption{PACS: leave-one-domain-out test accuracy under DomainBed-style training-domain validation. Held-out domain in the column header (A: Art, C: Cartoon, P: Photo, S: Sketch). Baseline numbers from~\citet{gulrajani2021search}; CRO is our runs. Best per column in bold; ties bolded jointly.}
\label{tab:pacs}
\centering
\small
\begin{tabular}{lccccc}
\toprule
\textbf{Algorithm}   & \textbf{A}           & \textbf{C}           & \textbf{P}           & \textbf{S}           & \textbf{Avg}         \\
\midrule
ERM                  & 84.7 $\pm$ 0.4       & 80.8 $\pm$ 0.6       & 97.2 $\pm$ 0.3       & 79.3 $\pm$ 1.0       & 85.5                 \\
IRM                  & 84.8 $\pm$ 1.3       & 76.4 $\pm$ 1.1       & 96.7 $\pm$ 0.6       & 76.1 $\pm$ 1.0       & 83.5                 \\
GroupDRO             & 83.5 $\pm$ 0.9       & 79.1 $\pm$ 0.6       & 96.7 $\pm$ 0.3       & 78.3 $\pm$ 2.0       & 84.4                 \\
Mixup                & 86.1 $\pm$ 0.5       & 78.9 $\pm$ 0.8       & \textbf{97.6} $\pm$ 0.1 & 75.8 $\pm$ 1.8    & 84.6                 \\
MLDG                 & 85.5 $\pm$ 1.4       & 80.1 $\pm$ 1.7       & 97.4 $\pm$ 0.3       & 76.6 $\pm$ 1.1       & 84.9                 \\
CORAL                & \textbf{88.3} $\pm$ 0.2 & 80.0 $\pm$ 0.5    & 97.5 $\pm$ 0.3       & 78.8 $\pm$ 1.3       & 86.2                 \\
MMD                  & 86.1 $\pm$ 1.4       & 79.4 $\pm$ 0.9       & 96.6 $\pm$ 0.2       & 76.5 $\pm$ 0.5       & 84.6                 \\
DANN                 & 86.4 $\pm$ 0.8       & 77.4 $\pm$ 0.8       & 97.3 $\pm$ 0.4       & 73.5 $\pm$ 2.3       & 83.6                 \\
CDANN                & 84.6 $\pm$ 1.8       & 75.5 $\pm$ 0.9       & 96.8 $\pm$ 0.3       & 73.5 $\pm$ 0.6       & 82.6                 \\
MTL                  & 87.5 $\pm$ 0.8       & 77.1 $\pm$ 0.5       & 96.4 $\pm$ 0.8       & 77.3 $\pm$ 1.8       & 84.6                 \\
SagNet               & 87.4 $\pm$ 1.0       & 80.7 $\pm$ 0.6       & 97.1 $\pm$ 0.1       & \textbf{80.0} $\pm$ 0.4 & \textbf{86.3}    \\
ARM                  & 86.8 $\pm$ 0.6       & 76.8 $\pm$ 0.5       & 97.4 $\pm$ 0.3       & 79.3 $\pm$ 1.2       & 85.1                 \\
VREx                 & 86.0 $\pm$ 1.6       & 79.1 $\pm$ 0.6       & 96.9 $\pm$ 0.5       & 77.7 $\pm$ 1.7       & 84.9                 \\
RSC                  & 85.4 $\pm$ 0.8       & 79.7 $\pm$ 1.8       & \textbf{97.6} $\pm$ 0.3 & 78.2 $\pm$ 1.2    & 85.2                 \\
\textbf{CRO (ours)}  & 82.8 $\pm$ 0.5       & \textbf{81.0} $\pm$ 0.7 & 97.5 $\pm$ 1.1       & 77.0 $\pm$ 4.2       & 84.6                 \\
\bottomrule
\end{tabular}
\end{table}

Unlike on Colored MNIST, no baseline collapses on PACS, and CRO does not lead the average---SagNet ($86.3$) and CORAL ($86.2$) sit at the top, with CRO at $84.6$ in line with the broad baseline cluster, leading Cartoon ($81.0$) and matching the Photo cluster ($97.5$, within $0.1$ of the Mixup/RSC tie at $97.6$). We read this under the structural lens of \Cref{thm:multi-id}: PACS shifts mix abduction shift (which CRO is built for) with partial shifts in the deduction map---rendering style on a sketch is not only a confounder for what's depicted but also a partial determinant of $f_Y$ (silhouettes that work as a sketch read as a different category as a photo). When deduction invariance holds only approximately, the parameterized set of plausible targets is a worse cover of the true target, and the worst-case-optimal predictor trades headroom against caution. PACS therefore reads as a robustness check rather than as a setting where CRO's structural assumptions cleanly apply.

\begin{figure}[t]
\centering
\includegraphics[width=0.55\linewidth]{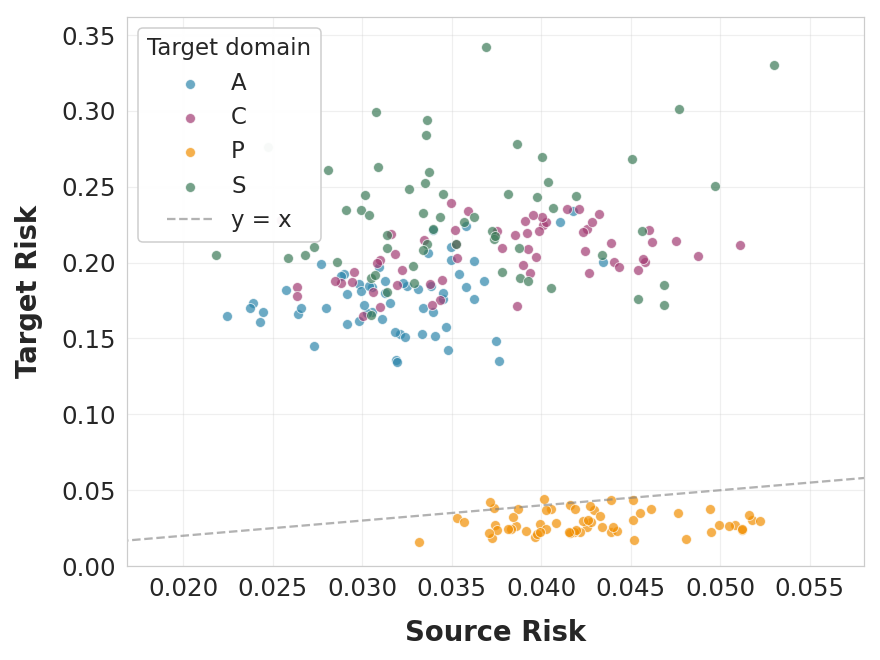}
\caption{Source vs.\ target risk across CRO checkpoints on PACS, by held-out domain. Source risk concentrates in a narrow $0.02$--$0.05$ band; target risk varies over an order of magnitude wider band depending on the held-out domain.}
\label{fig:pacs-source-vs-test}
\end{figure}
\paragraph{Source vs.\ target risk on PACS.} \Cref{fig:pacs-source-vs-test} is the PACS analog of \Cref{fig:source-vs-test}: source-validation risk vs.\ held-out target risk across CRO checkpoints, one color per held-out domain. The picture is qualitatively different from Colored MNIST. Source risk concentrates in a narrow $0.02$--$0.05$ band across all cells---the featurizer fits the three source domains tightly---while target risk varies over a much wider band (roughly $0.02$ to $0.34$) depending on the held-out domain. The held-out-domain identity dominates the spread: Photo as target sits below $y{=}x$ at $\approx 0.03$ target risk because real photographs are closer to the average of the other three source mixes; Sketch as target sits highest at $\approx 0.15$--$0.30$ because abstract line drawings depart most from photographic priors. At fixed source risk, target risk varies by $\sim$$0.2$ within a cell, so source-validation selection is much less informative on PACS than on the Colored MNIST source-tracking cells---consistent with the deduction-shift reading of the cross-domain behavior on PACS.

\paragraph{OfficeHome.} OfficeHome~\citep{venkateswara2017deep} contains 65 categories across 4 domains (Art, Clipart, Product, Real). \Cref{tab:officehome} reports leave-one-domain-out test accuracy under the standard DomainBed protocol.

\begin{table}[t]
\caption{OfficeHome: leave-one-domain-out test accuracy under DomainBed-style training-domain validation. Held-out domain in the column header (A: Art, C: Clipart, P: Product, R: Real). Baseline numbers from~\citet{gulrajani2021search}; CRO is our runs averaged over three seeds. Best per column in bold.}
\label{tab:officehome}
\centering
\small
\begin{tabular}{lccccc}
\toprule
\textbf{Algorithm}   & \textbf{A}           & \textbf{C}           & \textbf{P}           & \textbf{R}           & \textbf{Avg}         \\
\midrule
ERM                  & 61.3 $\pm$ 0.7       & 52.4 $\pm$ 0.3       & 75.8 $\pm$ 0.1       & 76.6 $\pm$ 0.3       & 66.5                 \\
IRM                  & 58.9 $\pm$ 2.3       & 52.2 $\pm$ 1.6       & 72.1 $\pm$ 2.9       & 74.0 $\pm$ 2.5       & 64.3                 \\
GroupDRO             & 60.4 $\pm$ 0.7       & 52.7 $\pm$ 1.0       & 75.0 $\pm$ 0.7       & 76.0 $\pm$ 0.7       & 66.0                 \\
Mixup                & 62.4 $\pm$ 0.8       & 54.8 $\pm$ 0.6       & 76.9 $\pm$ 0.3       & 78.3 $\pm$ 0.2       & 68.1                 \\
MLDG                 & 61.5 $\pm$ 0.9       & 53.2 $\pm$ 0.6       & 75.0 $\pm$ 1.2       & 77.5 $\pm$ 0.4       & 66.8                 \\
CORAL                & \textbf{65.3} $\pm$ 0.4 & 54.4 $\pm$ 0.5    & 76.5 $\pm$ 0.1       & 78.4 $\pm$ 0.5       & 68.7                 \\
MMD                  & 60.4 $\pm$ 0.2       & 53.3 $\pm$ 0.3       & 74.3 $\pm$ 0.1       & 77.4 $\pm$ 0.6       & 66.3                 \\
DANN                 & 59.9 $\pm$ 1.3       & 53.0 $\pm$ 0.3       & 73.6 $\pm$ 0.7       & 76.9 $\pm$ 0.5       & 65.9                 \\
CDANN                & 61.5 $\pm$ 1.4       & 50.4 $\pm$ 2.4       & 74.4 $\pm$ 0.9       & 76.6 $\pm$ 0.8       & 65.8                 \\
MTL                  & 61.5 $\pm$ 0.7       & 52.4 $\pm$ 0.6       & 74.9 $\pm$ 0.4       & 76.8 $\pm$ 0.4       & 66.4                 \\
SagNet               & 63.4 $\pm$ 0.2       & 54.8 $\pm$ 0.4       & 75.8 $\pm$ 0.4       & 78.3 $\pm$ 0.3       & 68.1                 \\
ARM                  & 58.9 $\pm$ 0.8       & 51.0 $\pm$ 0.5       & 74.1 $\pm$ 0.1       & 75.2 $\pm$ 0.3       & 64.8                 \\
VREx                 & 60.7 $\pm$ 0.9       & 53.0 $\pm$ 0.9       & 75.3 $\pm$ 0.1       & 76.6 $\pm$ 0.5       & 66.4                 \\
RSC                  & 60.7 $\pm$ 1.4       & 51.4 $\pm$ 0.3       & 74.8 $\pm$ 1.1       & 75.1 $\pm$ 1.3       & 65.5                 \\
\textbf{CRO (ours)}  & 64.8 $\pm$ 0.7       & \textbf{54.9} $\pm$ 1.2 & \textbf{77.9} $\pm$ 1.7 & \textbf{79.4} $\pm$ 0.7 & \textbf{69.3}    \\
\bottomrule
\end{tabular}
\end{table}

CRO leads the OfficeHome average ($69.3$ vs.\ CORAL $68.7$, Mixup/SagNet $68.1$) and tops three of the four cells (Clipart, Product, Real); on Art, CORAL leads by $0.5$ points ($65.3$ vs.\ $64.8$). The $65$-class structure of OfficeHome spreads the abduction signal across many fine-grained categories, so the per-domain prior $\pi_s(\tilde U)$ that CRO models is rich enough to give a meaningful adversary search space without entering the over-parameterized regime of \Cref{fig:rank-vs-test}; the structural assumption (deduction invariance) is also a closer fit than on PACS, because the four OfficeHome domains differ primarily in image style rather than in what is being depicted.

\subsection{Empirical validation of the structural premises}
\label{sec:empirical-validation}

To check that the structural conditions underlying CRO are simultaneously \emph{satisfiable} on a real estimator stack---rather than to argue that they hold in nature---we construct a moral-questions DGP in the spirit of the running example, and report four diagnostics tied to \Cref{ass:invariant-deduction,ass:linear-multi,ass:closure-multi}. The SCM in \Cref{ass:confounded-multi} is fixed by construction here (we choose the data-generating mechanism); the diagnostics below test whether the parameterized factorization closes empirically on $\phi(X)$, whether per-domain abduction maps actually differ, whether the deduction map is invariant across domains, and whether a transplant basis $L$ exists.

\paragraph{Construction.} Each record has unobserved moral values $U$ coarsened to a triple $\tilde U \in \{-1,0,+1\}^3$ ($|\tilde\cU|=27$) along the three axes of \Cref{tab:moral-axes}; $X$ is one of the $680$ \emph{high-ambiguity} binary-choice ethical scenarios from MoralChoice~\citep{scherrer2023evaluating}; $Y \in \{A, B\}$ is the chosen action elicited from \texttt{claude-haiku-4-5-20251001}, frozen, under a values-conditioned prompt encoding $\tilde U$. Three domains share the scenario library, the topic-conditional $P(\text{topic}\mid \tilde U)$, and the response mechanism $P(Y\mid X, \tilde U)$, and differ only in the per-axis moral-values prior $\pi_s(\tilde U)$ (\Cref{tab:domain-priors}). On the pooled $42{,}000$-record set ($14{,}000$/domain) we fit the abduction map $A_s$ by per-domain multinomial logistic regression of $\tilde U$ on $\phi(X)$, and the deduction map $D$ by per-$\tilde u$ binary logistic regression of $Y$ on $\phi(X)$ pooled across all domains, with $\phi(X) = \mathrm{RoBERTa\text{-}base\;CLS}(X)$~\citep{liu2019roberta}.

\paragraph{Factorization closure.} \Cref{tab:factorization-closure} reports per-domain quality of the factorized prediction $\widehat P^s(y_1\mid x) = \softmax(A_s^\top \phi(x))^\top \sigmoid(D^\top \phi(x))$ against the empirical labels and against an unconstrained binary logistic regression $X\to Y$. Calibration is exact to three decimals across all three domains and the per-record probability correlation with direct LR is $0.93$--$0.96$. The accuracy gap to direct LR ranges from $0$ pp on Pragmatist (factorized $=$ direct $= 0.809$) to $\approx 5$ pp on Idealist and Centrist, the price of constraining the predictor to the abduction--deduction factorization rather than letting it be an unconstrained black box.

\begin{table}[h]
\centering
\caption{Factorization closure. The factorized prediction $\widehat P^s(y_1\mid x) = \softmax(A_s^\top \phi(x))^\top \sigmoid(D^\top \phi(x))$ is compared per-domain to empirical labels and to an unconstrained $X\to Y$ logistic regression. Predicted mean matches empirical mean to three decimals across all three domains.}
\label{tab:factorization-closure}
\small
\begin{tabular}{llcccc}
\toprule
\textbf{Domain} & \textbf{Predictor} & \textbf{Pred mean} & \textbf{Brier} & \textbf{Acc} & \textbf{LogLoss} \\
\midrule
\multirow{2}{*}{Pragmatist} & Factorized & 0.465 & 0.171 & 0.809 & 0.526 \\
                            & Direct LR  & 0.459 & 0.148 & 0.809 & 0.470 \\
\midrule
\multirow{2}{*}{Idealist}   & Factorized & 0.250 & 0.142 & 0.774 & 0.448 \\
                            & Direct LR  & 0.250 & 0.123 & 0.823 & 0.395 \\
\midrule
\multirow{2}{*}{Centrist}   & Factorized & 0.337 & 0.188 & 0.708 & 0.559 \\
                            & Direct LR  & 0.337 & 0.170 & 0.761 & 0.514 \\
\bottomrule
\end{tabular}
\end{table}

\paragraph{Domain shift in abduction subspaces.} The abduction maps $A_s$ have stable rank $\approx 15$ across domains and are positioned distinctly in column space. \Cref{tab:principal-angles} reports pairwise principal angles between $\mathrm{col}(A_s)$ and $\mathrm{col}(A_{s'})$. The median $\sim$$65^\circ$ is stable across data scales (same numbers at $N=2{,}000$ and $N=14{,}000$ per domain), so the spread is a property of the data-generating process rather than sample noise: the three abduction maps carve out distinct preferred directions, giving the adversary three meaningfully separated ``corners'' to interpolate among.

\begin{table}[h]
\centering
\caption{Pairwise principal angles between the column spaces of the per-domain abduction maps $A_s$, in degrees. The $\sim$$65^\circ$ median is stable across data scales, indicating that the spread of abduction subspaces is a DGP property rather than sample noise.}
\label{tab:principal-angles}
\small
\begin{tabular}{lccc}
\toprule
\textbf{Pair} & \textbf{Min} & \textbf{Median} & \textbf{Max} \\
\midrule
Pragmatist vs.\ Idealist & 22 & 66 & 89 \\
Pragmatist vs.\ Centrist & 21 & 67 & 89 \\
Idealist   vs.\ Centrist & 19 & 67 & 90 \\
\bottomrule
\end{tabular}
\end{table}

\paragraph{Deduction invariance.} When a $(s, \tilde u)$ cell has enough samples to estimate the per-domain deduction column $D_s[:, \tilde u]$ cleanly, it matches the pooled $D[:, \tilde u]$ to column-cosine $1.000$. The median column-cosine across all valid columns is dragged down ($\approx 0.69$--$0.72$) only by sparsely-sampled cells where the per-domain LR is dominated by $L_2$ regularization rather than data, sharpening to near-$1$ as $N$ grows---consistent with sample-size noise rather than genuine domain-dependence in $D$. Deduction invariance (\Cref{ass:invariant-deduction}) holds.

\paragraph{Existence of a transplant basis $L$.} We seek $L\in\R^{d\times r}$ satisfying $A_s^\top L = I_r$ for $s\in\{\text{prag}, \text{ideal}, \text{cent}, \star\}$ and $D^\top L = 0$ simultaneously, with the Hausdorff penalty controlling where the transplant lands relative to source-pool support. \Cref{tab:L-feasibility} sweeps the penalty weight $\lambda_H$. The biorthogonality residual sits at the rank-deficiency floor ($\approx 1$ per $A$, driven by a single rare cell with $n=6$ records), independent of $\lambda_H$; $D$-orthogonality is achieved at machine epsilon ($\|D^\top L\|_F < 5\times 10^{-4}$) across the sweep. Increasing $\lambda_H$ shrinks the Hausdorff term from $998$ to $377$ at modest cost in biorthogonality, mapping out a clean trade-off curve. $L$ exists across the full sweep, so \Cref{ass:closure-multi} is non-vacuous on this DGP.

\begin{table}[h]
\centering
\caption{Existence of a transplant basis $L$ under varying Hausdorff penalty $\lambda_H$. Biorthogonality residual is per-$A$ Frobenius distance from $I_r$ (rank-deficiency floor $\approx 1$, set by a single rare $(s,\tilde u)$ cell). $L$ exists across the full sweep; the Hausdorff penalty selects which feasible $L$ is picked.}
\label{tab:L-feasibility}
\small
\begin{tabular}{cccc}
\toprule
$\lambda_H$ & \textbf{Bio.\ residual / $A$} & $\|L\|_F$ & \textbf{Hausdorff} \\
\midrule
$0$       & $1.06$ & $282.8$ & $998$ \\
$0.01$    & $1.06$ & $282.5$ & $984$ \\
$0.1$     & $1.06$ & $280.3$ & $900$ \\
$1.0$     & $1.06$ & $267.7$ & $662$ \\
$10$      & $1.19$ & $210.3$ & $377$ \\
\bottomrule
\end{tabular}
\end{table}

\paragraph{Summary.} Across the four diagnostics: (i) the abduction--deduction factorization closes empirically (calibration exact, $89$--$91\%$ argmax-agreement with direct LR), (ii) abduction maps are domain-distinct (median $\sim$$65^\circ$ pairwise principal angles), (iii) deduction invariance holds (per-domain $D_s$ matches pooled $D$ to cosine $1.000$ in well-sampled cells), and (iv) a transplant basis $L$ exists at the rank-deficiency floor with $D$-orthogonality essentially free and a clean Hausdorff trade-off. The structural premises that CRO assumes are simultaneously satisfiable on the empirical estimators, with all residuals consistent with finite-sample noise. We emphasize the scope of this conclusion: the moral-questions DGP is constructed so that the structural premises hold by design, and these diagnostics show that the estimator stack recovers them rather than that the premises hold in nature; the experiment is a self-consistency check on the parameterization, not evidence that the same premises hold on PACS, OfficeHome, or other real-world benchmarks.

\section{Conclusion}
\label{sec:conclusion}

Target risk under an unseen distribution is not pinned down by source data: many abduction--deduction factorizations are consistent with the same source-trained predictor (the abduction--deduction entanglement, \Cref{eq:AD-decomp}), and they disagree about how predictions shift under regime change. Representation transplant---a class of low-rank linear operations on a frozen pretrained representation---turns this entanglement into a parameterization. The headline use, presented in \Cref{sec:transplant}, is multi-domain generalization: under deduction invariance and closure under transplant, a learner--adversary game (CRO; \Cref{sec:algorithm}) attains best worst-case target risk over the parameterized set of plausible target distributions, with invariant representations recovered as a corollary (\Cref{cor:invariant-rep}). On the diagnostic Colored MNIST cell the resulting classifier holds up where established baselines collapse; on the real-world PACS and OfficeHome benchmarks, where deduction invariance is only approximately satisfied, CRO is competitive with the strongest baselines (leading on OfficeHome and matching the broad baseline cluster on PACS). A separately constructed moral-questions DGP (\Cref{sec:empirical-validation}) checks that the structural conditions CRO leans on are simultaneously satisfiable on real estimators, with residuals consistent with sample noise. In short: representation transplant gives a small, structured parameterization of abduction--deduction non-identifiability on top of a frozen pretrained representation, and the CRO algorithm built on it attains minimax-optimal worst-case target risk under the structural assumptions of \Cref{sec:algorithm}.

\paragraph{Limitations.} The closure assumption (\Cref{ass:closure-multi}) is the load-bearing condition throughout the paper. It plays a role analogous to positivity/coverage in standard causal inference and is closely connected to disentanglement assumptions in causal representation learning~\citep{scholkopf2021toward,locatello2019challenging,vonkugelgen2021self,ahuja2023interventional}; it is not directly testable from observational data, and its empirical enforcement is via penalty (\Cref{eq:match}) on source pairs only. Closure to the target regime is assumed, not enforced. Deduction invariance (\Cref{ass:invariant-deduction}) is a sharper assumption than it may appear: real-world shifts often mix abduction shift with partial shifts in $f_Y$ itself (e.g., rendering style on PACS is not purely a confounder), so the parameterized set of plausible targets is a tight cover only when the structural match-up holds. This is reflected empirically: CRO's headroom is largest on Colored MNIST, where the assumption matches the data-generating process by construction, and shrinks on real-world benchmarks where the match-up is approximate. The minimax-optimality result of \Cref{thm:multi-id} additionally requires infinite source data and global convergence of the gradient-based searches inside CRO; \Cref{thm:robustness} and \Cref{prop:finite-sample} relax these to a linear $O(\epsilon)$ degradation under approximate closure and a $O(1/\sqrt N)$ statistical plug-in error, but tight finite-sample constants for the softmax/sigmoid-linear-in-$\phi$ parameterization are left to future work.

\paragraph{Outlook.} The construction in this paper is best read as a way to parameterize causal models on top of pretrained foundation representations. Two directions follow naturally. \emph{Sequence models and causal structure among tokens.} The setup here treats $X$ as a single input with a representation $\phi(X)$. Lifting it to sequence models---where the causal relationships sit among the tokens themselves and can carry richer structure than a single confounder $\tilde U$---opens partial-transportability questions to language modeling and other sequential domains. \emph{Supervised domain adaptation via transplant fine-tuning.} When a small labeled target sample is available, the same abduction--deduction parameterization admits a different statistical regime: rather than retraining $(\phi, h)$ on the target, one can fix $\phi$ and fit the transplant parameters $(\{\tilde A_s\}, \tilde L)$ jointly to source and target data, with potential statistical advantages over full target re-training. This is mechanically related to the parameter-efficient finetuning regime of ReFT~\citep{wu2024reft}; the new content here is that the optimized parameters carry a causal-structural meaning (an abduction--deduction factorization), so the same operator family also delivers a transportability guarantee rather than just an empirical adaptation. The relative statistical efficiency of transplant-fine-tuning vs.\ full target re-training is an open question.

\bibliography{references}

\appendix
\section{Proofs}
\label{app:proofs}

This appendix gives formal statements and proofs for the claims in \Cref{sec:transplant,sec:algorithm}. We work throughout under the parameterizations \cref{eq:abduction-param,eq:deduction-param} as exact identities (writing ``$=$'' rather than ``$\approx$''); approximation error from finite-rank representation fits propagates linearly through every identity below.

\subsection{Existence of a transplant basis}
\label{app:proofs-L-exists}

The cross-domain transplant operator (\Cref{def:T-multi}) requires a matrix $L\in\R^{d\times r}$ that simultaneously left-inverts every abduction map and lies in the left-null-space of the deduction map. The next proposition characterizes when such an $L$ exists.

\begin{proposition}[Existence of a transplant basis]\label{prop:L-exists}
Let $A_1,\dots,A_K,A_\star,D\in\R^{d\times r}$ and let
$M:=[A_1\;\cdots\;A_K\;\;A_\star\;\;D]\in\R^{d\times(K+2)r}$. The system
\begin{equation}\label{eq:bio-system-app}
A_s^\top L=I_r\;\;\forall s\in\{1,\dots,K,\star\},\qquad D^\top L=0,
\end{equation}
admits a solution $L\in\R^{d\times r}$ if and only if the following \emph{kernel condition} holds: every
$(a_1,\dots,a_K,a_\star,c)\in\R^r\times\cdots\times\R^r$ with
$\sum_{s=1}^K A_s a_s + A_\star a_\star + D c = 0$ satisfies $\sum_{s=1}^K a_s + a_\star = 0$.
A sufficient condition is $\rank(M)=(K+2)r$: the columns of $[A_1\;\cdots\;A_K\;A_\star\;D]$ are jointly linearly independent. The dimension constraint $d\ge(K+2)r$ is \emph{necessary} for $\rank(M)=(K+2)r$ but not sufficient---under $d\ge(K+2)r$ the rank attains its maximum at generic-position parameters and can fail to do so at non-generic ones (e.g.\ overlapping column spans of $\{A_s,D\}$).
When the kernel condition holds and $\rank(M)<d$, the solution set is an affine subspace of $\R^{d\times r}$ of dimension $r\,(d-\rank(M))$.
\end{proposition}

\begin{proof}
Stack $L=[\ell_1,\dots,\ell_r]\in\R^{d\times r}$ column by column. The constraints in \cref{eq:bio-system-app} read, for each $j\in\{1,\dots,r\}$,
\[
M^\top\ell_j=t_j,\qquad t_j:=\begin{bmatrix}e_j\\\vdots\\e_j\\0_r\end{bmatrix}\in\R^{(K+2)r},
\]
where $e_j$ denotes the $j$-th canonical basis vector of $\R^r$, appearing in the $K+1$ abduction blocks and zero in the deduction block. By the fundamental theorem of linear algebra, $M^\top\ell_j=t_j$ has a solution iff $t_j\in\range(M^\top)=(\ker M)^\perp$, equivalently $\langle t_j,z\rangle=0$ for every $z\in\ker M$. Writing $z=(a_1,\dots,a_K,a_\star,c)$,
\[
\langle t_j,z\rangle=\sum_{s=1}^K e_j^\top a_s+e_j^\top a_\star+0_r^\top c=\sum_{s=1}^K a_{s,j}+a_{\star,j},
\]
so the condition is $\sum_{s=1}^K a_s+a_\star=0$. The sufficient condition $\rank(M)=(K+2)r$ implies $\ker M=\{0\}$, in which case the constraint is vacuous. The dimension count follows from $\dim\range(M^\top)=\rank(M)$ and the fact that the solution set of $M^\top\ell_j=t_j$ is an affine $(d-\rank(M))$-dimensional translate of $\ker(M^\top)$. \Cref{tab:L-feasibility} reports empirical feasibility for $K=3$ at the rank-deficiency floor.
\end{proof}

\begin{remark}[Underdetermination and the role of the Hausdorff penalty]\label{rem:L-underdet}
When $\rank(M)<d$, the solution set of \cref{eq:bio-system-app} is an affine subspace of dimension $r\,(d-\rank(M))$. The Hausdorff penalty in \cref{eq:match} selects a specific element of this affine subspace---the parameterization treats $L$ as a free parameter and uses the Hausdorff term to anchor the transplanted representations near the empirical source supports.
\end{remark}

\subsection{Action of the transplant operator on abduction/deduction readouts}
\label{app:multi-bio}

\begin{lemma}[Action of $T_{s\to s'}$ on the readouts]\label{lem:multi-bio}
Under \Cref{ass:linear-multi,ass:closure-multi}, for every $s,s'\in\{1,\dots,K,\star\}$ and every $W\in\R^d$,
\begin{equation}\label{eq:multi-bio}
A_{s'}^\top T_{s\to s'}W = A_s^\top W,\qquad
D^\top T_{s\to s'}W = D^\top W.
\end{equation}
Moreover, $T_{s\to s'}$ is unipotent ($\det T_{s\to s'}=1$), hence a linear bijection of $\R^d$.
\end{lemma}

\begin{proof}
By \Cref{def:T-multi}, $T_{s\to s'}=I_d+L(A_s^\top-A_{s'}^\top)$. Using $A_{s'}^\top L=A_s^\top L=I_r$ and $D^\top L=0$,
\begin{align*}
A_{s'}^\top T_{s\to s'}W
&= A_{s'}^\top W+(A_{s'}^\top L)(A_s^\top-A_{s'}^\top)W
 = A_{s'}^\top W+(A_s^\top-A_{s'}^\top)W=A_s^\top W,\\
D^\top T_{s\to s'}W
&= D^\top W+(D^\top L)(A_s^\top-A_{s'}^\top)W = D^\top W.
\end{align*}
For the determinant, write $T_{s\to s'}=I_d+K$ with $K:=L(A_s^\top-A_{s'}^\top)$. The matrix $K$ has rank at most $r$ and satisfies $KL=L(A_s^\top-A_{s'}^\top)L=L(I_r-I_r)=0$, so $\range(K)\subseteq\ker(K)$, i.e., $K$ is nilpotent ($K^2=0$). Hence all eigenvalues of $K$ are $0$ and $\det(I_d+K)=1$.
\end{proof}

\subsection{Invariant representation: proof of \texorpdfstring{\cref{cor:invariant-rep}}{Corollary~\ref{cor:invariant-rep}}}
\label{app:proofs-cor-invariant-rep}

We first state the distributional-match condition in formal terms; \Cref{sec:invariant-rep} introduces it informally.

\begin{assumption}[Distributional match across domains]\label{ass:dist-match}
Let $L$ satisfy \Cref{ass:closure-multi} and pick any $A\in\R^{d\times r}$ with $A^\top L=I_r$ (e.g.\ $A=A_s$ for any $s\in\{1,\dots,K,\star\}$ works under \Cref{ass:closure-multi}). Decompose
\[
\phi(X)=W=W_A+W_D,\qquad W_A:=L A^\top\phi(X),\quad W_D:=(I_d-L A^\top)\phi(X).
\]
For every $s,s'\in\{1,\dots,K,\star\}$ and every $w_D\in\range(I_d-L A^\top)$, the conditional distribution of $W_A$ given $W_D=w_D$ in domain $s$ matches the pushforward under $T_{s\to s'}$ of the corresponding conditional in domain $s'$:
\begin{equation*}
P^s\!\big(W_A\mid W_D=w_D\big) \;=\; (T_{s\to s'})_\#\,P^{s'}\!\big(W_A\mid W_D=w_D\big),
\end{equation*}
viewing $T_{s\to s'}$ as a linear bijection of $\R^d$ (\Cref{lem:multi-bio}) restricted to the appropriate fibers. Under \Cref{ass:finite-support} the relevant measures are discrete on $\phi(\cX)\subset\R^d$, and the pushforward is the relabeling of finite-support point masses induced by $T_{s\to s'}$.
\end{assumption}

\begin{remark}[$T_{s\to s'}$ preserves the deduction component]\label{rem:T-preserves-WD}
Since $T_{s\to s'}-I_d=L(A_s^\top-A_{s'}^\top)$ takes values in $\range(L)$, which is annihilated by $I_d-LA^\top$, the operator $T_{s\to s'}$ leaves $W_D$ invariant: $(I_d-LA^\top)T_{s\to s'}w=(I_d-LA^\top)w$ for all $w$. Consequently $T_{s\to s'}$ restricts to a bijection of each $W_A$-fiber $\{w_D+w_A:w_A\in\range(L)\}$, acting on the $w_A$-coordinate as the affine map $w_A\mapsto w_A+L(A_s^\top-A_{s'}^\top)(w_D+w_A)$. The linear part of this restriction equals $\bigl(I+L(A_s^\top-A_{s'}^\top)\bigr)|_{\range(L)}=I|_{\range(L)}$ (using $A_s^\top L=A_{s'}^\top L=I_r$), so the fiberwise restriction is a translation of $\range(L)$ — a bijection of the $W_A$-fiber. Under \Cref{ass:finite-support} the pushforward in \Cref{ass:dist-match} is the relabeling of the discrete $W_A$-mass at fixed $W_D=w_D$ induced by this bijection.
\end{remark}

\begin{proof}[Proof of \Cref{cor:invariant-rep}]
Fix $s,s'\in\{1,\dots,K,\star\}$ and $w_D\in\range(I_d-L A^\top)$. By the law of total probability (and \Cref{ass:finite-support}, so the integral is a sum over the discrete $W_A$-fiber at $W_D=w_D$),
\begin{equation}\label{eq:cor-invariant-step1}
P^s(y_1\mid W_D=w_D)
=\sum_{w_A} P^s(y_1\mid W=w_D+w_A)\,P^s(w_A\mid w_D).
\end{equation}
By \Cref{ass:closure-multi}, for $\phi$-almost every $w=w_D+w_A$ in $\supp_{P^s}(\phi)$ the transplanted vector $T_{s\to s'}w$ lies in $\supp_{P^{s'}}(\phi)$; combining the abduction--deduction entanglement \cref{eq:AD-entanglement} for both domains with the parameterizations \cref{eq:abduction-param,eq:deduction-param} and \Cref{lem:multi-bio},
\begin{align*}
P^s(y_1\mid W=w)
&= \softmax(A_s^\top w)^\top\sigmoid(D^\top w)\\
&= \softmax(A_{s'}^\top T_{s\to s'}w)^\top\sigmoid(D^\top T_{s\to s'}w)
 = P^{s'}(y_1\mid W=T_{s\to s'}w),
\end{align*}
so the summand in \cref{eq:cor-invariant-step1} satisfies $P^s(y_1\mid W=w_D+w_A)=P^{s'}(y_1\mid W=T_{s\to s'}(w_D+w_A))$. Substituting and applying \Cref{ass:dist-match} as a relabeling on the (finite) $W_A$-fiber,
\begin{align*}
P^s(y_1\mid W_D=w_D)
&= \sum_{w_A} P^{s'}\!\big(y_1\mid W=T_{s\to s'}(w_D+w_A)\big)\,P^s(w_A\mid w_D)\\
&= \sum_{w_A'} P^{s'}\!\big(y_1\mid W=w_D+w_A'\big)\,P^{s'}(w_A'\mid w_D)\\
&= P^{s'}(y_1\mid W_D=w_D).\qedhere
\end{align*}
\end{proof}

\subsection{Formal definition of the plausible target set}
\label{app:plausible-set-formal}

\Cref{thm:multi-id} ranges over the set $\cP^\star(P^1,\dots,P^K)$ of target distributions \emph{entailed} by the source distributions and the structural assumptions \Cref{ass:confounded-multi,ass:finite-support,ass:invariant-deduction,ass:linear-multi,ass:closure-multi}.

\begin{definition}[Plausible target set]\label{def:plausible-set}
Fix source distributions $P^1,\dots,P^K$ on $\cX\times\cY$ and a frozen representation $\phi:\cX\to\R^d$. The \emph{plausible target set entailed by $P^1,\dots,P^K$} is
\begin{equation}\label{eq:plausible-set}
\cP^\star(P^1,\dots,P^K)
:=\Big\{\,\tilde P^\star\in\cP(\cX\times\cY)\;:\;
  \exists\,(\tilde A_1,\dots,\tilde A_K,\tilde A_\star,\tilde D,\tilde L)\in\Theta
  \text{ such that }\,\Pi(\tilde\theta;P^1,\dots,P^K,\tilde P^\star)\text{ holds}\,\Big\},
\end{equation}
where $\Theta$ is the parameter domain and $\Pi$ collects the structural constraints, defined as follows:
\begin{enumerate}
\item[\textnormal{(P1)}]\textnormal{(Domain and norm bound.)} $\tilde A_s,\tilde D,\tilde L\in\R^{d\times r}$ for every $s\in\{1,\dots,K,\star\}$, with operator norms bounded by a fixed constant $B$: $\|\tilde A_s\|_{\mathrm{op}},\|\tilde D\|_{\mathrm{op}},\|\tilde L\|_{\mathrm{op}}\le B$. The empirical analog of this bound is the operator-norm penalty controlled by \texttt{cro\_lambda\_op} (see \Cref{tab:hp-search-cro}); the softmax/sigmoid parameterizations \cref{eq:abduction-param,eq:deduction-param} are invariant under translating $\tilde A_s$ by $\1_r u^\top$ for any $u\in\R^d$ (a softmax gauge), so the bound is imposed after fixing this gauge (e.g., zero-mean columns of $\tilde A_s$).
\item[\textnormal{(P2)}]\textnormal{(Linear abduction.)} The transplant dimension is identified with the coarsening support size, $r=|\tilde\cU|$; for every $x\in\cX$ and every $s\in\{1,\dots,K,\star\}$,
\[
\vec{\tilde P}^s(\bm{\tilde u}\mid x)\;:=\;\softmax(\tilde A_s^\top\phi(x))\;\in\;\R^r,
\]
read as a probability vector indexed by $\tilde u\in\tilde\cU$ in a fixed ordering (\Cref{ass:linear-multi}).
\item[\textnormal{(P3)}]\textnormal{(Invariant deduction.)} $\tilde D\in\R^{d\times r}$ does not depend on $s$; for every $x\in\cX$ and every $\tilde u\in\tilde\cU$,
\[
\tilde P(y_1\mid x,\tilde u)\;:=\;\bigl[\sigmoid(\tilde D^\top\phi(x))\bigr]_{\tilde u},
\]
the $\tilde u$-th entry of the sigmoid output, read as a conditional probability (\Cref{ass:invariant-deduction}).
\item[\textnormal{(P4)}]\textnormal{(Source-fit.)} For every $s\in\{1,\dots,K\}$ and every $x\in\supp_{P^s}(X)$,
\[
P^s(y_1\mid x)\;=\;\softmax(\tilde A_s^\top\phi(x))^\top\,\sigmoid(\tilde D^\top\phi(x)).
\]
\item[\textnormal{(P5)}]\textnormal{(Biorthogonality.)} $\tilde A_s^\top\tilde L=I_r$ for every $s\in\{1,\dots,K,\star\}$ and $\tilde D^\top\tilde L=0$ (\Cref{ass:closure-multi}).
\item[\textnormal{(P6)}]\textnormal{(Closure.)} For every $s,s'\in\{1,\dots,K,\star\}$ and every $w\in\R^d$, $P^s(\phi(X)=w)>0\iff P^{s'}(\phi(X)=\tilde T_{s\to s'}w)>0$, where $\tilde T_{s\to s'}:=I_d+\tilde L(\tilde A_s^\top-\tilde A_{s'}^\top)$ (\Cref{ass:closure-multi}).
\item[\textnormal{(P7)}]\textnormal{(Target consistency.)} The marginal $\tilde P^\star(X)$ has $\supp_{\tilde P^\star}(X)=\supp_{P^\star}(X)$, and the conditional satisfies $\tilde P^\star(y_1\mid x)=\softmax(\tilde A_\star^\top\phi(x))^\top\,\sigmoid(\tilde D^\top\phi(x))$ for every $x\in\supp_{\tilde P^\star}(X)$.
\end{enumerate}
\end{definition}

Conditions (P1)--(P3) and (P5)--(P6) are the structural assumptions of \Cref{ass:linear-multi,ass:invariant-deduction,ass:closure-multi}; (P4) imposes that the parameterization fits the source distributions; (P7) defines the candidate target distribution from the parameters. (P4) together with the abduction--deduction entanglement \cref{eq:AD-entanglement} pins the inner product $\tilde A_s^\top\phi(x)\cdot\tilde D^\top\phi(x)$ to the source prediction $P^s(y_1\mid x)$ but leaves the individual factors underdetermined---this is the multi-domain abduction--deduction entanglement.

\begin{remark}[Empirical version]\label{rem:plausible-set-empirical}
\Cref{def:plausible-set} uses populations $P^s$. With finite samples, \Cref{alg:CRO} works with empirical $\hat P^s$, and (P4) is enforced by the source-fit penalty $\mathrm{Fit}$ in \cref{eq:loss-grand} with weight $\lambda_1$; (P6) is enforced by the Hausdorff penalty $\mathrm{Closure}$ in \cref{eq:match} with weight $\lambda_2$. As $\lambda_1,\lambda_2\to\infty$ the empirical adversary's feasible set converges to the plausible set $\cP^\star(\hat P^1,\dots,\hat P^K)$ in \Cref{def:plausible-set} with the empirical distributions in (P4)--(P6).
\end{remark}

\subsection{Target-side consistency identity}
\label{app:target-consistency}

\begin{proposition}[Target consistency]\label{prop:transplant-consistency-multi}
Under \Cref{ass:confounded-multi,ass:invariant-deduction,ass:linear-multi,ass:closure-multi}, for every $s'\in\{1,\dots,K\}$ and every $x_0^\star\in\cX$ with $P^\star(\phi(X)=\phi(x_0^\star))>0$,
\begin{equation}\label{eq:target-consistency-multi}
P^\star(y_1\mid X=x_0^\star)
=P^{s'}\!\big(y_1\mid X\in\phi^{-1}(T_{\star\to s'}\,\phi(x_0^\star))\big).
\end{equation}
This is \cref{eq:w-star-0} in the body.
\end{proposition}

\begin{proof}
Fix $s'\in\{1,\dots,K\}$ and $x_0^\star\in\cX$ with $P^\star(\phi(X)=\phi(x_0^\star))>0$. Write $T:=T_{\star\to s'}$ and $w:=\phi(x_0^\star)$. By \Cref{ass:closure-multi} applied to $\{\star,s'\}$, there exists $x'\in\cX$ with $\phi(x')=Tw$ and $P^{s'}(\phi(X)=\phi(x'))>0$. Under \Cref{ass:linear-multi,ass:invariant-deduction}, $P^{s'}(y_1\mid x)$ depends on $x$ only through $\phi(x)$ (both the abduction parameterization $\softmax(A_{s'}^\top\phi(x))$ and the deduction parameterization $\sigmoid(D^\top\phi(x))$ are functions of $\phi(x)$ alone), so $P^{s'}(y_1\mid x)$ is constant on $\phi$-fibers and $P^{s'}(y_1\mid X\in\phi^{-1}(Tw))=P^{s'}(y_1\mid x')$ for any representative $x'\in\phi^{-1}(Tw)\cap\supp_{P^{s'}}(X)$. Combining the abduction--deduction entanglement \cref{eq:AD-entanglement} for $P^{s'}$ at $x'$ with the parameterizations \cref{eq:abduction-param,eq:deduction-param},
\begin{align*}
P^{s'}(y_1\mid x')
&= \softmax(A_{s'}^\top Tw)^\top\sigmoid(D^\top Tw)
 = \softmax(A_\star^\top w)^\top\sigmoid(D^\top w),
\end{align*}
where the second equality is \Cref{lem:multi-bio}. Applying \cref{eq:AD-entanglement} for $P^\star$ at $x_0^\star$ together with the target abduction parameterization $\softmax(A_\star^\top\phi(x))=\vec P^\star(\bm{\tilde u}\mid x)$ and the shared (invariant) deduction parameterization $\sigmoid(D^\top\phi(x))=\vec P^\star(y_1\mid x,\bm{\tilde u})$, the right-hand side equals $P^\star(y_1\mid x_0^\star)$. Hence $P^\star(y_1\mid x_0^\star)=P^{s'}(y_1\mid x')$, which is \cref{eq:target-consistency-multi}.
\end{proof}

\subsection{Proof of \texorpdfstring{\cref{thm:multi-id}}{Theorem~\ref{thm:multi-id}}}
\label{app:proofs-multi-id}

The proof has three steps. First, the target consistency identity (\Cref{prop:transplant-consistency-multi}) ties the unobserved target to the sources. Second, we identify the adversary's optimal solution set with $\cP^\star(P^1,\dots,P^K)$ from \Cref{def:plausible-set}. Third, we invoke the min-max convergence statement of \Cref{thm:min-max-conv} below to conclude that the learner's iterates converge to the minimax optimum over $\cP^\star$.

\begin{proof}[Proof of \Cref{thm:multi-id}]
\noindent\emph{Step 1 (Target consistency).} By \Cref{prop:transplant-consistency-multi}, for every $s'\in\{1,\dots,K\}$ and every $x_0^\star\in\supp_{P^\star}(X)$, \cref{eq:target-consistency-multi} holds: the unobserved target $P^\star(Y\mid X)$ at any in-support point equals a transplant of an observed source conditional.

\smallskip
\noindent\emph{Step 2 (Plausible-set characterization.)} Define
\[
\cF(P^1,\dots,P^K):=\Big\{\,(\tilde g^\star,\{\tilde A_s\}_{s=1}^K,\tilde A_\star,\tilde D,\tilde L)\;:\;\text{(P1)--(P6) hold}\,\Big\}
\]
and recall from \Cref{def:plausible-set} that each $\tilde\theta\in\cF$ induces a candidate target distribution $\tilde P^\star(\tilde\theta)$ via (P7). We claim
\[
\big\{\tilde P^\star(\tilde\theta):\tilde\theta\in\cF(P^1,\dots,P^K)\big\}
=\cP^\star(P^1,\dots,P^K),
\]
i.e., the parameter family ranged over by the adversary in line~4 of \Cref{alg:CRO} (in the limit $\lambda_1,\lambda_2\to\infty$) is exactly the plausible target set of \Cref{def:plausible-set}. The inclusion ``$\subseteq$'' is by construction. For ``$\supseteq$'', let $\tilde P^\star\in\cP^\star$ and let $\tilde\theta=(\tilde A_1,\dots,\tilde A_K,\tilde A_\star,\tilde D,\tilde L)$ be the witness in \Cref{def:plausible-set}; this witness satisfies (P7), which requires \Cref{ass:closure-multi} to hold at $s'=\star$---i.e., target-side closure is assumed, not constructed. Setting $\tilde g^\star(W):=\softmax(\tilde A_\star^\top W)^\top\sigmoid(\tilde D^\top W)$ gives a feasible point with $\tilde P^\star(\tilde\theta)=\tilde P^\star$ through (P7). The empirical adversary in line~4 enforces (P4)--(P6) at the sources but cannot enforce (P7); target-side closure is the load-bearing assumption that lets the source-only witness extend to a target-side one.

\smallskip
\noindent\emph{Step 3 (Min-max convergence.)} Under \Cref{ass:finite-support}, the surrogate target risk $\tilde P^\star\mapsto R_{\tilde P^\star}(h)$ is continuous, and for each $\tilde P^\star$ the map $h\mapsto R_{\tilde P^\star}(h)$ is convex. Here ``$\cH$ is a convex function class'' is meant in the function-space sense: $\cH$ is a set of $[0,1]$-valued classification heads closed under finite convex combinations, $\lambda h_1+(1-\lambda)h_2\in\cH$ for $h_1,h_2\in\cH,\lambda\in[0,1]$. The running loss $\ell$ is the convex surrogate cross-entropy (\Cref{sec:intro}), which is pointwise convex in the predicted probability $\hat p=h(x)$; combined with $\cH$ closed under convex combinations and the surrogate evaluated on the finite support of \Cref{ass:finite-support}, this gives convexity of $h\mapsto R_{\tilde P^\star}(h)$ by pointwise propagation. (No restriction to a particular parameterization of $\cH$, e.g., linear-in-$\phi$, is required; non-convex parametric families such as deep nets fit this framework via their convex hull, equivalently stochastic mixtures.) The parameter space
\[
\Theta_B:=\Big\{(\tilde A_1,\dots,\tilde A_K,\tilde A_\star,\tilde D,\tilde L)\in(\R^{d\times r})^{K+3}\;:\;\text{(P1)--(P6) hold}\Big\}
\]
is bounded by the operator-norm constraint in (P1) and closed by the polynomial equality constraints in (P2)--(P6) (the support condition (P6) is closed because $X$ has finite support by \Cref{ass:finite-support}), hence compact in the finite-dimensional ambient space. The plausible set $\cP^\star(P^1,\dots,P^K)$ is the image of $\Theta_B$ under the continuous parameter-to-distribution map $\tilde\theta\mapsto\tilde P^\star(\tilde\theta)$ defined by (P7), hence compact. The hypotheses of \Cref{thm:min-max-conv} therefore hold. Applying that theorem with adversary oracle (line~4 of \Cref{alg:CRO}) and learner oracle (line~6) gives $M_t\to V^\star:=\min_{h\in\cH}\max_{\tilde P^\star\in\cP^\star}R_{\tilde P^\star}(h)$ as $t\to\infty$. By the lower-bound-at-the-limit step of \Cref{thm:min-max-conv}, the worst-case risk of iterate $h_t$ satisfies $\max_{\tilde P^\star\in\cP^\star}R_{\tilde P^\star}(h_t)\le M_{t+1}$. For the lower bound: since $h_t\in\cH$ is one specific element of $\cH$, its worst-case risk is at least the infimum of worst-case risks over $\cH$ — that is,
\[
\max_{\tilde P^\star\in\cP^\star}R_{\tilde P^\star}(h_t)\;\ge\;\inf_{h\in\cH}\max_{\tilde P^\star\in\cP^\star}R_{\tilde P^\star}(h)\;=\;V^\star,
\]
which is a definitional consequence of $V^\star$ being the minimax value and does \emph{not} require the convergence we are establishing. Sandwiching gives
\[
V^\star \;\le\; \max_{\tilde P^\star\in\cP^\star}R_{\tilde P^\star}(h_t)\;\le\; M_{t+1}\;\xrightarrow{t\to\infty}\;V^\star.
\]
Termination triggers as $\epsilon\to 0$ at $t=t(\epsilon)\to\infty$, so $\max_{\tilde P^\star\in\cP^\star}R_{\tilde P^\star}(h^\star)\to V^\star$ as $\epsilon\to 0$, which is the conclusion of \Cref{thm:multi-id}. The argmax over the plausible set $\cP^\star$ corresponds to the worst-case target distribution entailed by $(K+1)$-tuples of SCMs compatible with \Cref{ass:confounded-multi,ass:finite-support,ass:invariant-deduction,ass:linear-multi,ass:closure-multi}.
\end{proof}

\subsection{Min-max game and convergence}
\label{app:min-max}

The CRO algorithm (\Cref{alg:CRO}) is an instance of an iterative learner--adversary min-max game: at iteration $t$, the adversary supplies a risk-maximizing distribution $\hat P^\star_t\in\cP^\star$, the learner solves a minimax problem on the growing collection $\hat\cP^\star_t=\{\hat P^\star_1,\dots,\hat P^\star_t\}$, and the procedure terminates when the learner's worst-case risk on $\hat\cP^\star_t$ stops increasing. The convergence statement below was invoked in Step 3 of the proof of \Cref{thm:multi-id}; the argument follows the structure of \citet[Appendix C]{jalaldoust2024partial}, adapted to the transplant-parameterized adversary.

\begin{theorem}[Min-max convergence]\label{thm:min-max-conv}
Let $\cP$ be a compact set of probability distributions on $\cX\times\cY$, let $\cH$ be a convex hypothesis class, and let $\ell$ be a loss with $h\mapsto R_P(h)$ convex and $P\mapsto R_P(h)$ continuous. Consider the iteration: starting from $\hat\cP_0=\emptyset$ and any $h_0\in\cH$, for $t\ge 1$ set
\[
P_t\in\argmax_{P\in\cP}R_P(h_{t-1}),\qquad
\hat\cP_t:=\hat\cP_{t-1}\cup\{P_t\},\qquad
h_t\in\argmin_{h\in\cH}\,\max_{P\in\hat\cP_t}R_P(h),
\]
and let $M_t:=\max_{P\in\hat\cP_t}R_P(h_t)$. Assume each argmax/argmin attains its global optimum. Then $M_t$ is non-decreasing in $t$, and
\[
\lim_{t\to\infty}M_t \;=\; \min_{h\in\cH}\,\max_{P\in\cP}R_P(h).
\]
For every $\epsilon>0$, the stopping rule $|M_t-M_{t-1}|\le\epsilon$ triggers in finitely many iterations, and $M_t$ at termination is within $\epsilon$ of the minimax value.
\end{theorem}

\begin{proof}[Proof sketch]
\emph{Monotonicity.} At round $t$, $\hat\cP_t\supseteq\hat\cP_{t-1}$, so $M_t=\min_h\max_{P\in\hat\cP_t}R_P(h)\ge\min_h\max_{P\in\hat\cP_{t-1}}R_P(h)=M_{t-1}$ (the inner max over a larger set is no smaller). Hence $M_t\ge M_{t-1}$.

\emph{Boundedness.} $M_t\le\max_{P\in\cP}R_P(h_t)\le\sup_{h,P}R_P(h)<\infty$ by compactness and continuity. So $M_t$ converges to some $M_\infty$.

\emph{Lower bound at the limit.} By construction $M_{t+1}\ge\max_{P\in\hat\cP_{t+1}}R_P(h_{t+1})\ge R_{P_{t+1}}(h_t)=\max_{P\in\cP}R_P(h_t)$, so $M_\infty\ge\liminf_{t\to\infty}\max_{P\in\cP}R_P(h_t)\ge\min_h\max_{P\in\cP}R_P(h)$.

\emph{Upper bound at the limit.} $M_t=\min_{h\in\cH}\max_{P\in\hat\cP_t}R_P(h)\le\min_h\max_{P\in\cP}R_P(h)$ since $\hat\cP_t\subseteq\cP$. Together, $M_\infty=\min_h\max_{P\in\cP}R_P(h)$, the minimax value.

\emph{Finite termination.} The increasing bounded sequence $M_t$ has $|M_t-M_{t-1}|\to 0$, so any positive tolerance $\epsilon$ is eventually achieved.
\end{proof}

\begin{remark}[Application in CRO]\label{rem:min-max-cro}
\Cref{thm:min-max-conv} applies with $\cP=\cP^\star(P^1,\dots,P^K)$ (\Cref{def:plausible-set}). The algorithm itself manipulates the empirical plausible set $\hat\cP^\star$ constructed from $\hat P^1,\dots,\hat P^K$; under the infinite-source-data hypothesis of \Cref{thm:multi-id}, the empirical source distributions coincide with the population $P^s$ and the empirical plausible set coincides with $\cP^\star$, so the convergence target $V^\star$ is the population minimax value over $\cP^\star$. The adversary oracle in line~4 of \Cref{alg:CRO} approximates $\argmax_{P\in\cP^\star}R_P(h_{t-1})$ by gradient ascent on $\hat{\mathrm{Obj}}$ over the parameters $(\tilde g^\star,\{\tilde A_s\},\tilde L)$; under the ``each gradient-based search attains its global optimum'' hypothesis of \Cref{thm:multi-id} the oracle is exact. The learner oracle in line~6 solves $\min_{h\in\cH}\max_{P\in\hat\cP^\star}R_P(h)$ over the finite collection $\hat\cP^\star$, which is a standard finite-domain minimax problem.
\end{remark}

\subsection{Robustness and finite-sample bounds}
\label{app:robustness}

This subsection proves the two robustness statements from \Cref{sec:robustness}: \Cref{thm:robustness} (graceful degradation under approximate closure) and \Cref{prop:finite-sample} (finite-sample plug-in error). The plausible target set $\cP^\star_\epsilon$ under the $\epsilon$-relaxed closure of \Cref{def:eps-closure} is defined exactly as in \Cref{def:plausible-set}, with condition \textnormal{(P6)} replaced by its $\epsilon$-relaxed analog: for every $s,s'\in\{1,\dots,K,\star\}$ and every $w$ in $\supp_{P^s}(\phi(X))$, there exists $w'\in\supp_{P^{s'}}(\phi(X))$ with $\|w'-\tilde T_{s\to s'}w\|\le\epsilon$ (and symmetrically).

\subsubsection{Proof of \texorpdfstring{\Cref{thm:robustness}}{Theorem 3.9}}
\label{app:robustness-proof}

We split the bound into two steps: a coupling step that shows every $\tilde P^\star\in\cP^\star_\epsilon$ admits a partner $\tilde P^\star_0\in\cP^\star_0$ within $W_1$-distance $\epsilon$, and a stability step that propagates the coupling to the risk.

\smallskip
\emph{Step 1: Coupling.} Let $\tilde P^\star\in\cP^\star_\epsilon$ with witness $(\tilde A_1,\dots,\tilde A_K,\tilde A_\star,\tilde D,\tilde L)$. Define $\tilde P^\star_0\in\cP^\star_0$ as the distribution that, for each transplanted source representation $\tilde T_{s\to\star}\phi(x)$, snaps to the closest target-support point $\phi(x')$ with $\|\phi(x')-\tilde T_{s\to\star}\phi(x)\|\le\epsilon$ (which exists by $\epsilon$-relaxed closure), with mass preserved. Under \Cref{ass:finite-support} both distributions are discrete with finite support; the snapping induces a deterministic mass-preserving rearrangement of $\tilde P^\star$ onto $\tilde P^\star_0$ in which each atom moves by at most $\epsilon$, so
\[
W_1(\tilde P^\star,\tilde P^\star_0)\;\le\;\epsilon.
\]

\smallskip
\emph{Step 2: Stability.} The hypothesis that $h\in\cH$ is $L_h$-Lipschitz with values in $[\delta,1-\delta]$ gives, for cross-entropy $\ell(\hat p,y)=-\log\hat p_y$,
\[
|\ell(h(w),y)-\ell(h(w'),y)|\;\le\;\frac{1}{\delta}\,|h(w)-h(w')|\;\le\;\frac{L_h}{\delta}\|w-w'\|,
\]
so $w\mapsto\ell(h(w),y)$ is $(L_h/\delta)$-Lipschitz for each fixed $y$. By Kantorovich--Rubinstein duality on the discrete support,
\[
|R_{\tilde P^\star}(h)-R_{\tilde P^\star_0}(h)|\;\le\;\frac{L_h}{\delta}\cdot W_1(\tilde P^\star,\tilde P^\star_0)\;\le\;\frac{L_h}{\delta}\cdot\epsilon.
\]
Taking the supremum over $\tilde P^\star\in\cP^\star_\epsilon$ and the corresponding $\tilde P^\star_0\in\cP^\star_0$,
\[
\sup_{\tilde P^\star\in\cP^\star_\epsilon}R_{\tilde P^\star}(h)\;\le\;\sup_{\tilde P^\star_0\in\cP^\star_0}R_{\tilde P^\star_0}(h)\;+\;\frac{L_h}{\delta}\,\epsilon,
\]
and taking $\min_h$ on both sides yields $V^\star_\epsilon\le V^\star_0+(L_h/\delta)\epsilon$. The lower bound $V^\star_\epsilon\ge V^\star_0$ is immediate from $\cP^\star_0\subseteq\cP^\star_\epsilon$.\hfill$\square$

\subsubsection{Sketch of \texorpdfstring{\Cref{prop:finite-sample}}{Proposition 3.10}}
\label{app:finite-sample-sketch}

We give the decomposition; tight bracketing-entropy constants for the softmax/sigmoid-linear-in-$\phi$ class are standard and out of scope for this paper.

Write $\widehat{V}:=\min_{h\in\cH}\max_{\hat P\in\hat\cP^\star}R_{\hat P}(h)$ for the empirical minimax value attained by \Cref{alg:CRO} in the exact-oracle setting, and decompose
\[
\max_{\tilde P^\star\in\cP^\star}R_{\tilde P^\star}(\hat h)
\;=\;V^\star_0
\;+\;\underbrace{\bigl[\max_{\tilde P^\star\in\cP^\star}R_{\tilde P^\star}(\hat h)-\widehat{V}\bigr]}_{(\mathrm{a})\text{ statistical}}
\;+\;\underbrace{\bigl[\widehat{V}-V^\star_0\bigr]}_{(\mathrm{b})\text{ approx.\ + closure}}.
\]
Term (a) bounds the population worst case of the empirical minimizer by the empirical worst case at the same head; by Step~3 of the \Cref{thm:multi-id} proof the parameter space $\Theta_B$ in (P1) is compact in operator norm, and the parameter-to-distribution map $\tilde\theta\mapsto\tilde P^\star(\tilde\theta)$ is continuous. A uniform-convergence argument over $\Theta_B$ at scale $\eta$ then gives, with probability $1-\delta$,
\[
\sup_{\tilde\theta\in\Theta_B}|R_{P^\star(\tilde\theta)}(\hat h)-R_{\hat P^\star(\tilde\theta)}(\hat h)|\;\le\;C_1\sqrt{\frac{\log\mathcal{N}(\Theta_B,\eta)+\log(1/\delta)}{N}}\;+\;\eta,
\]
where $C_1$ depends on the cross-entropy Lipschitz constant $1/\delta$ from \Cref{thm:robustness} and the operator-norm bound $B$. Term (b) decomposes further into the optimization gap $\epsilon_{\mathrm{opt}}$ of the inner adversary/learner searches relative to their global optima (this is zero under the oracle hypothesis of \Cref{thm:multi-id}, and quantified by the gradient-method analyses cited there in practice), and the empirical Hausdorff residual $\epsilon_{\mathrm{Haus}}=\mathrm{Closure}(\hat A_s,\hat L)$ which contributes to the gap between the empirical plausible set $\hat\cP^\star$ and its population counterpart $\cP^\star$ via the same $W_1$-coupling argument as in \Cref{thm:robustness}. Combining gives \cref{eq:finite-sample-bound}. The bracketing-entropy bound for softmax/sigmoid-linear-in-$\phi$ heads with explicit constants (e.g.\ via standard arguments for VC-major classes or covering numbers of bounded linear maps) is orthogonal to the paper's contribution and left to future work.\hfill$\square$

\section{Experimental details}
\label{app:experiments}

We document the implementation, training procedure, hyperparameter search, datasets, compute, and reproducibility protocol for CRO. Settings not noted here follow the DomainBed defaults of~\citet{gulrajani2021search}.

\subsection{Codebase and framework}
\label{app:codebase}

CRO is implemented on top of the DomainBed framework~\citep{gulrajani2021search}, registered as a new \texttt{Algorithm} subclass alongside ERM and the baselines reported in our tables. We preserve DomainBed's sweep protocol, dataset loaders, train/holdout splits, and model-selection criteria.

\subsection{Architectures}
\label{app:architectures}

\paragraph{Image datasets (PACS, OfficeHome).} The featurizer $\phi$ is ResNet-50 pretrained on ImageNet with AugMix (\texttt{timm} checkpoint \texttt{resnet50.ram\_in1k}). BatchNorm parameters are frozen during both training phases. The final classification head $h$ is reinitialized as a single linear layer on top of $\phi$.

\paragraph{MNIST-style datasets (Colored MNIST).} The featurizer is the standard 4-layer convolutional \texttt{MNIST\_CNN} used by DomainBed.

\paragraph{Inputs and augmentation.} Image inputs are resized to $224\times 224$ with ImageNet mean/std normalization. Training-time augmentation: \texttt{RandomResizedCrop(224, scale=(0.7,1.0))}, \texttt{RandomHorizontalFlip}, \texttt{ColorJitter(0.3,0.3,0.3,0.3)}, and \texttt{RandomGrayscale}. No augmentation on Colored MNIST beyond DomainBed defaults.

\subsection{Two-phase training}
\label{app:two-phase}

CRO training proceeds in two phases.

\paragraph{Phase 1 --- pretrain.} Standard ERM-style training of the $(\phi, h)$ pair on the union of source domains: 5000 SGD steps for image datasets, 500 steps for Colored MNIST. Phase-1 checkpoints are cached on disk keyed by $(\text{dataset},\,\text{test\_envs},\,\text{trial\_seed})$; subsequent hyperparameter samples within a sweep reuse the same checkpoint, so the pretrain cost is paid once per (dataset, test environment, trial). Phase-1 learning rate, batch size, and weight decay are fixed (not sampled) to enable cache hits: LR $5\times 10^{-5}$ (image) or $10^{-3}$ (MNIST), batch size $32$ (image) or $64$ (MNIST), weight decay $0$.

\paragraph{Phase 2 --- CRO refinement.} The adversary $\tilde g^\star$ is initialized as a copy of the Phase-1 head $h$. The outer loop of \Cref{alg:CRO} alternates: (i) update the adversary using the transplant-dimension-$r$ reparameterized operator (parameters $(\{\tilde A_s\}, \tilde L)$); (ii) update the learner $h$. Termination is when the running gap of the worst-case risk drops below $\epsilon_{\text{gap}}$ (sampled as \texttt{cro\_gap\_eps}) or \texttt{cro\_max\_outer\_iters} is reached. Each outer iteration runs a fixed number of inner adversary and learner epochs (sampled per-dataset; see \Cref{tab:hp-search-phase2}).

\subsection{Hyperparameter search and model selection}
\label{app:hp-search}

For each dataset and each held-out test environment we sample \textbf{20 hyperparameter configurations $\times$ 3 trial seeds}, yielding $240$ runs per 4-domain benchmark (180 for 3-domain Colored MNIST). All jobs use DomainBed's \texttt{-{}-single\_test\_envs} mode (one held-out target per run). Model selection follows the DomainBed default: the best hyperparameter configuration per cell is picked using \emph{training-domain validation accuracy} only, so target labels are never observed during selection.

\Cref{tab:hp-search-cro} reports the search ranges (log-uniform unless noted), and \Cref{tab:hp-search-phase2} reports per-dataset Phase-2 compute ranges, chosen so the total Phase-2 work is comparable across datasets of different scales.

\begin{table}[h]
\centering
\caption{CRO hyperparameter search ranges. Sampled log-uniformly unless noted. The body's $\lambda_1$ corresponds to \texttt{cro\_lambda\_consist} (source-CE term in \cref{eq:obj}); $\lambda_2$ to \texttt{cro\_lambda\_hausdorff} (support match, \cref{eq:match}); \texttt{cro\_lambda\_op} is an additional operator-norm penalty on $(\{\tilde A_s\}, \tilde L)$ that stabilizes the adversary updates.}
\label{tab:hp-search-cro}
\small
\begin{tabular}{ll}
\toprule
\textbf{Hyperparameter} & \textbf{Range} \\
\midrule
\texttt{cro\_rank} ($r$, transplant dimension) & $\{4,\, 8,\, 16\}$ \\
\texttt{cro\_gap\_eps} & $[10^{-3},\, 10^{-2}]$ \\
\texttt{cro\_adv\_lr} & $[10^{-3},\, 10^{-2}]$ \\
\texttt{cro\_learner\_lr} & $[10^{-3},\, 10^{-2}]$ \\
\texttt{cro\_lambda\_hausdorff} ($\lambda_2$, support match) & $[0.1,\, 1.0]$ \\
\texttt{cro\_lambda\_consist} ($\lambda_1$, source CE) & $[10,\, 100]$ \\
\texttt{cro\_lambda\_op} (operator-norm penalty) & $[0.1,\, 1.0]$ \\
\texttt{cro\_val\_frac} & $0.1$ (fixed) \\
\bottomrule
\end{tabular}
\end{table}

\begin{table}[h]
\centering
\caption{Per-dataset Phase-2 compute ranges. Each cell is a search range over the listed values.}
\label{tab:hp-search-phase2}
\small
\begin{tabular}{lccccc}
\toprule
\textbf{Tier} & \textbf{Outer iters} & \textbf{Adv epochs} & \textbf{Lrn epochs} & \textbf{Adv batch} & \textbf{Lrn batch} \\
\midrule
Colored MNIST & $\{5,\,10\}$ & $\{5,\,10\}$ & $\{5,\,10\}$ & $\{32,\,64\}$ & $\{32,\,64\}$ \\
PACS & $\{3,\,5\}$ & $\{3,\,5\}$ & $\{3,\,5\}$ & $\{128,\,256\}$ & $\{128,\,256\}$ \\
OfficeHome & $\{5,\,10\}$ & $\{5,\,10\}$ & $\{5,\,10\}$ & $\{128,\,256\}$ & $\{128,\,256\}$ \\
\bottomrule
\end{tabular}
\end{table}

\subsection{Datasets}
\label{app:datasets}

\begin{table}[h]
\centering
\caption{Datasets used. We follow DomainBed's standard loaders and 80/20 train/holdout split per domain.}
\label{tab:datasets}
\small
\begin{tabular}{llcc}
\toprule
\textbf{Dataset} & \textbf{Domains} & \textbf{Classes} & \textbf{Images} \\
\midrule
Colored MNIST~\citep{arjovsky2019invariant} & 3 ($\rho=+0.9,\, +0.8,\, -0.9$) & 2 & 70{,}000 \\
PACS~\citep{li2017deeper} & 4 (Art, Cartoon, Photo, Sketch) & 7 & 9{,}991 \\
OfficeHome~\citep{venkateswara2017deep} & 4 (Art, Clipart, Product, Real) & 65 & ${\sim}15{,}500$ \\
\bottomrule
\end{tabular}
\end{table}

\subsection{Compute}
\label{app:compute}

All experiments run on a shared workstation with 8 $\times$ NVIDIA H100 (80\,GB). Each benchmark uses 2 GPUs in parallel with up to 4 jobs per GPU (8 concurrent slots), with per-job memory $7$--$12$\,GB. Datasets and Phase-1 caches are staged in a \texttt{/dev/shm} RAM tier to remove disk I/O as a bottleneck on the shared box; per-process thread caps (\texttt{OMP\_NUM\_THREADS=16}, \texttt{TORCH\_NUM\_THREADS=16}) prevent CPU oversubscription.

\begin{table}[h]
\centering
\caption{Approximate sweep wall-clock under the 8-slot parallelism described above.}
\label{tab:compute}
\small
\begin{tabular}{lcc}
\toprule
\textbf{Sweep} & \textbf{Runs} & \textbf{Wall-clock} \\
\midrule
Colored MNIST & 180 & ${\sim}1$\,h \\
PACS & 240 & ${\sim}17$\,h \\
OfficeHome & 240 & ${\sim}10$--$12$\,h \\
\bottomrule
\end{tabular}
\end{table}

Per-job wall-clock: $<\,1$\,min on Colored MNIST, ${\sim}10$--$15$\,min on PACS, ${\sim}20$\,min on OfficeHome.

\subsection{Reproducibility}
\label{app:reproducibility}

Trial seeds are $\{0, 1, 2\}$, paired with hyperparameter seeds drawn from DomainBed's standard sweep. The per-job random seed is a deterministic function of $$(\text{dataset},\,\text{algorithm},\,\text{test\_env},\,\text{trial\_seed},\,\text{hp\_seed})$$, so any individual run reproduces from those five identifiers. The Phase-1 pretrain cache is keyed by $(\text{dataset},\,\text{test\_envs},\,\text{trial\_seed})$ and produces bit-identical checkpoints under repeated runs, isolating Phase-2 variation. The hyperparameter ranges (\Cref{tab:hp-search-cro,tab:hp-search-phase2}), search procedure, and DomainBed sweep configuration above are sufficient to reproduce the reported numbers up to seed-level variance.

\end{document}